\documentclass[times ,review,10pt]{elsarticle}
\usepackage{amssymb}
\usepackage{amsmath}
\usepackage{algorithm}
\usepackage{algpseudocode}

\usepackage[utf8]{inputenc}
\usepackage{booktabs} 
\usepackage{longtable}
\usepackage{geometry} 
\usepackage{multirow}
\usepackage{booktabs}     
\usepackage{makecell}      
\usepackage{graphicx}     
\journal{Pattern Recognition}
\begin{document}

\title{BPE: Behavioral Profiling Ensemble}  
\author[1]{Yanxin Liu}
\ead{12024213042@stu.ynu.edu.cn}

\author[1]{Yunqi Zhang}
\ead{yunqizhang@ynu.edu.cn}

\affiliation[1]{organization={Yunnan University},
 }

\begin{abstract}
In the field of machine learning, ensemble learning is widely recognized as a pivotal strategy for pushing the boundaries of predictive performance. Traditional static ensemble methods typically assign weights by treating each base learner as a whole, thereby overlooking that individual models exhibit varying competence across different regions of the instance space. Dynamic Ensemble Selection (DES) was introduced to address this limitation. However, both static and dynamic approaches predominantly rely on inter-model differences as the basis for integration; this inter-model perspective neglects models' intrinsic characteristics and often requires heavy reliance on reference sets for competence estimation. We propose the Behavioral Profiling Ensemble (BPE) framework, which introduces a model-centric integration paradigm. Unlike traditional methods, BPE constructs an intrinsic behavioral profile $\mathcal{P}_k$ for each model and derives aggregation weights from the deviation between a model's response to a test instance and its established profile; in this work, we instantiate $\mathcal{P}_k$ with entropy-based summary statistics (e.g., mean and variance). Extensive experiments on 42 real-world datasets show that BPE-derived algorithms outperform state-of-the-art DES baselines, increasing predictive accuracy while reducing computational and storage overhead.
\end{abstract}

\begin{keyword}
Ensemble Learning, Behavioral Profiling, Dynamic Ensemble Selection
\end{keyword}

\maketitle
\newtheorem{definition}{Definition}
\newtheorem{theorem}{Theorem}
\newtheorem{lemma}{Lemma}      
\newtheorem{assumption}{Assumption} 
\newtheorem{corollary}{Corollary}  
\newtheorem{proof}{Proof}   
\section{Introduction}\label{1}

In the field of machine learning, ensemble learning improves predictive performance by combining multiple base learners and is widely used for complex classification problems. Whether in data science competitions or high-stakes industrial applications such as medical diagnosis and financial risk management \cite{vanderlaan2007super, timmermann2006forecast}, ensemble models typically exhibit superior generalization capability and robustness compared to individual models by exploiting diversity among learners \cite{krogh1995neural, dietterich2000ensemble}. For mainstream probability-based ensemble integration, strategies are generally categorized into two paradigms: global static ensembles, such as simple averaging, median averaging, or rank averaging \cite{kittler1998combining, ho1994decision}; and meta-learning-based methods, such as Stacking, which learn the combination weights of base models by training a meta-classifier \cite{wolpert1992stacked}.

Despite their substantial success, many global fusion rules adopt a single combination rule that is applied across the entire data distribution. However, in practical scenarios, different models often manifest distinct strengths in different regions of the instance space. A global integration perspective can therefore overlook variations in local competence. To address this issue, Dynamic Ensemble Selection (DES) and Dynamic Classifier Selection (DCS) methods have emerged \cite{cruz2018dynamic}. Many DES/DCS methods follow a neighborhood-based validation paradigm: during the inference phase, they retrieve similar samples from a reference set and dynamically assign weights based on the models' historical performance on these local neighbors \cite{woods1997combination, ko2008dynamic}. Nevertheless, such reliance on an external reference set poses challenges in deployment and may limit robustness in certain contexts. As a result, in practice (e.g., data science competitions and real-world deployments), static ensembles, especially simple probability averaging/weighting, remain more prevalent due to their simplicity and ease of deployment.

To overcome these limitations, we propose a shift in perspective: moving from external reference matching to leveraging models' intrinsic properties. This paradigm shift can be understood through an intuitive analogy. Traditional DES/DCS algorithms resemble ``resume screening'': when a new task (test sample) arrives, the system reviews historical records to find individuals who performed well on similar past tasks (reference neighbors). The drawback is that for novel or outlier instances, such historical records may become less informative. In contrast, our method functions as ``behavioral profiling'': rather than asking whether the individual has encountered similar tasks before, we observe inherent behavioral traits through lightweight stress testing to assign higher weights to the most reliable models.

Driven by this motivation, this paper proposes the Behavioral Profiling Ensemble (BPE), a validation-free dynamic ensemble framework based on model ``personality'' profiling. Here, validation-free means that BPE does not require a held-out validation split to learn aggregation weights, and it avoids reference-set-based neighbor retrieval at inference time. Specifically, we avoid explicit neighborhood search on a reference set and instead introduce stochastic perturbations in the feature space to construct an instantaneous ``testing field'' for each model. By observing the variations in the output distributions of different base models under perturbation (e.g., entropy variation), we can characterize the models' inherent behavioral characteristics \cite{gal2016dropout, lakshminarayanan2017simple}. This profiling effectively differentiates between models that are robustly confident and those that exhibit high uncertainty on the current sample, thereby serving as the basis for dynamic weighting. Extensive experiments on both simulation and real-world datasets demonstrate that BPE consistently outperforms traditional methods across numerous scenarios.

The main contributions of this paper are summarized as follows:
\begin{itemize}
    \item \textbf{A new ensemble learning framework:} This paper proposes a new ensemble learning framework, BPE, which is distinct from traditional DES/DCS. Specifically, BPE shifts the integration perspective from inter-model comparison to a model-centric one. It further introduces the notion of a model-specific behavioral profile and a deviation-based weighting principle that evaluates each model relative to itself.
    \item \textbf{Key building blocks and an effective instantiation:} This paper presents the basic building blocks of the BPE framework and proposes a simple instantiation, BPE-Entropy, which improves accuracy over traditional reference-set-based DES/DCS methods in many scenarios while reducing deployment-relevant storage and computational overhead.
\end{itemize}

The remainder of this paper is organized as follows: Section 2 reviews related work in ensemble learning. Section 3 introduces the BPE framework and an entropy-based implementation algorithm. Section 4 presents comprehensive simulation experiments to verify the effectiveness of the algorithm. Section 5 evaluates the performance using a wide range of real-world datasets. Finally, Section 6 discusses future directions for the BPE framework.

\section{Related Work}\label{2}

\subsection{Ensemble Learning Theory and Uncertainty Measures}
The core philosophy of ensemble learning lies in compensating for the inductive bias of a single learner by combining multiple hypotheses. The ``No Free Lunch'' theorems proposed by Wolpert and Macready \cite{wolpert1997no} state that no single learning algorithm can perform optimally across all possible data distributions, thereby motivating the use of diverse hypotheses in practice. To quantify the effectiveness of ensembles, Krogh and Vedelsby proposed the well-known ``Error-Ambiguity Decomposition'' theory \cite{krogh1995neural}; in their formulation, the ensemble generalization error can be expressed as the average error of the base models minus an ambiguity term that captures model diversity. This theory establishes the dual objectives of ensemble learning: maximizing diversity among models while pursuing individual accuracy. From a statistical perspective, Dietterich further noted that ensemble methods can smooth decision boundaries by averaging multiple hypotheses and often reduce variance compared to individual models \cite{dietterich2000ensemble}.

To accurately allocate weights in dynamic ensembles, it is essential to quantify a model’s sample-wise predictive uncertainty. Information entropy, as introduced by Shannon \cite{shannon1948mathematical}, is a fundamental measure of uncertainty for probability distributions and is widely used as a convenient proxy for confidence; however, low entropy does not necessarily imply correctness when models are miscalibrated. Grandvalet and Bengio showed in semi-supervised learning that entropy minimization can encourage confident predictions on unlabeled data \cite{grandvalet2004semi}. In modern machine learning, perturbation- and sampling-based approaches to uncertainty estimation have become mainstream. Deep Ensembles \cite{lakshminarayanan2017simple} demonstrated that epistemic uncertainty can be estimated by training and combining multiple independently initialized models. Similarly, Gal and Ghahramani showed that enabling dropout at inference time provides a Bayesian approximation \cite{gal2016dropout}; by aggregating multiple stochastic forward passes, one can estimate predictive uncertainty (e.g., via predictive entropy). Furthermore, Schapire’s margin theory suggests that improving classification margins—often reflected by larger gaps between top predicted probabilities—can benefit generalization \cite{schapire1998boosting}. Taken together, these studies motivate our BPE framework, which constructs a model-specific behavioral baseline from perturbation-induced variations in its predictive distribution and performs dynamic aggregation based on the magnitude of a model’s deviation from this baseline.

\subsection{Static Ensembles}
Early ensemble strategies predominantly adopted a static perspective, where the aggregation weights of base models remain fixed across the entire feature space. Depending on the generation mechanism of the base learners, these strategies are primarily categorized into homogeneous and heterogeneous ensembles; in this paper, we use ``homogeneous'' to refer to ensembles that generate base learners from a shared underlying model family (often decision trees) via a common mechanism (e.g., Bagging or Boosting).

In homogeneous ensembles, the Bagging algorithm proposed by Breiman \cite{breiman1996bagging} constructs diversity via bootstrap sampling; its prominent representative, Random Forests \cite{breiman2001random}, primarily enhances performance by mitigating variance. Conversely, Boosting paradigms aim to reduce bias by iteratively correcting residuals. Notable examples include the classic AdaBoost \cite{freund1997decision} and Gradient Boosting Machines (GBM) \cite{friedman2001greedy}, along with highly optimized implementations such as XGBoost \cite{chen2016xgboost}, LightGBM \cite{ke2017lightgbm}, and CatBoost \cite{prokhorenkova2018catboost}.

In heterogeneous ensembles that combine diverse algorithms (e.g., SVC \cite{cortes1995support}, KNN \cite{cover1967nearest}, and Logistic Regression \cite{cox1958regression}), the prevailing strategy involves static fusion rules. Kittler et al. systematized fundamental combination rules, including majority voting and simple averaging \cite{kittler1998combining}. To address the limited comparability of probability estimates across models (e.g., due to differing calibration), rank-based methods, such as the Borda Count applied by Ho et al. \cite{ho1994decision}, mitigate these differences through rank aggregation. Furthermore, the Stacking framework \cite{wolpert1992stacked} employs a meta-learner to learn a global (instance-invariant) combination of base models. Despite their widespread success, these static methods inherently overlook the locality of data distributions; a model with suboptimal global performance may exhibit superior competence within specific local subspaces.

\subsection{Dynamic Ensembles}
To address the rigidity of static ensembles, Dynamic Ensemble Selection and Dynamic Classifier Selection (DES/DCS) have been extensively studied \cite{cruz2018dynamic}. Their core philosophy is to estimate the competence of base classifiers for each test sample and then perform dynamic selection and/or weighting accordingly. Existing DES/DCS methods primarily evaluate models based on a ``Region of Competence'' (RoC), which is typically constructed from a reference set. Approaches predicated on local accuracy, such as LCA proposed by Woods et al. \cite{woods1997combination}, assess models by computing class-conditional accuracy within a local neighborhood. The KNORA series \cite{ko2008dynamic} adopts a more direct oracle-based strategy, with different variants determining model participation based on whether neighbors in the RoC are correctly classified. To circumvent the limitations of hard thresholds, RRC \cite{woloszynski2011probabilistic} utilizes Gaussian potential functions to assign weights based on the distance between the test sample and its reference neighbors, achieving a probabilistic competence assessment. MCB \cite{giacinto2001design} introduces the concept of a behavior knowledge space to measure reliability by comparing the consistency of classifier output profiles in local regions. Meta-learning-based methods, such as META-DES \cite{cruz2015meta}, transform dynamic selection into a binary classification problem, utilizing meta-features to train a meta-classifier for competence estimation. Furthermore, FIRE-DES++ \cite{cruz2019fire} incorporates online pruning strategies to enhance computational efficiency.

Despite their theoretical sophistication, these DES/DCS methods encounter significant challenges regarding their dependency on external reference sets for practical deployment. Fundamentally, existing DES methods adhere to a retrieval-based paradigm: the system is required to retain a (potentially large) reference set and execute computationally expensive nearest-neighbor searches for each test sample during inference. As data scales expand, time and space complexity typically increase, imposing non-trivial overhead. Moreover, in high-dimensional feature spaces, as noted by Beyer et al. \cite{Beyer1999NN}, the discriminative power of Euclidean distance can diminish (the ``curse of dimensionality''), resulting in retrieved ``neighbors'' that may be spurious and less informative for competence assessment. In summary, while DES offers a perspective of dynamic adaptation, its structural reliance on external reference sets and neighborhood retrieval constrains its scalability to larger and higher-dimensional datasets. This limitation motivates our exploration of a novel ensemble paradigm that achieves efficient and robust dynamic integration by constructing intrinsic behavioral profiles of the models. This perspective provides a new ensemble direction: rather than relying on an external reference set or absolute statistics, we use the deviation between a model’s response and its own behavioral baseline.

\section{The core intuition of BPE}\label{3}
In this section, we will provide some simple proofs to derive the core intuition of BPE.
\begin{definition}[Models and Notations]
Let $S$ be the set of all samples in a classification task with output space $V = \{1, 2, \dots, |V|\}$. We consider an ensemble of two models sharing the same input and output spaces:
\begin{itemize}
    \item Let $q$ denote the \textbf{Primary Model (PM)}.
    \item Let $q'$ denote the \textbf{Secondary Model (SM)}.
\end{itemize}
For any sample $s \in S$, let $q_{s}$ and $q'_{s}$ be the output probability vectors of the Primary and Secondary models, respectively. Let $q_{s,i}$ denote the probability assigned to class $i$ by model $q$, and $k_s \in V$ denote the true class label of sample $s$.
\end{definition}

\begin{definition}[Partition of Sample Space]
Based on the correctness of the predictions from the Primary and Secondary models, the entire sample set $S$ is partitioned into three mutually disjoint subsets ($S = T \cup F \cup N$):
\begin{itemize}
    \item \textbf{$T$:} Samples where the Primary Model predicts incorrectly, but the Secondary Model predicts correctly.
    $$T = \{s \in S \mid \arg\max_{i \in V}(q_{s,i}) \neq k_s \land \arg\max_{i \in V}(q'_{s,i}) = k_s\}$$
    \item \textbf{$F$:} Samples where the Primary Model predicts correctly, but the Secondary Model predicts incorrectly.
    $$F = \{s \in S \mid \arg\max_{i \in V}(q_{s,i}) = k_s \land \arg\max_{i \in V}(q'_{s,i}) \neq k_s\}$$
    \item \textbf{$N$:} Samples where both models predict correctly, or both predict incorrectly.
    $$N = \{s \in S \mid I(\arg\max q_s = k_s) = I(\arg\max q'_s = k_s)\}$$
\end{itemize}
\end{definition}

\begin{definition}[Theoretical Optimality of Ensemble]
Let $Q_{ens}$ be the fused predictor derived from the ensemble of $q$ and $q'$. We define the \textbf{Theoretical Optimality} of the system as the state where the ensemble correctly classifies all samples that are correctly classified by at least one of the individual models, as well as the mathematically correctable subset $N_{corr} \subset N$ where both models predict incorrectly.

Formally, the ensemble achieves Theoretical Optimality if and only if:
$$ \forall s \in T \cup F \cup N_{corr}, \quad \arg\max_{i \in V}(Q_{ens, s, i}) = k_s $$
This implies that the classification accuracy on the union set $T \cup F \cup N_{corr}$ is $100\%$.
\end{definition}

\begin{lemma}[Exchange Condition]\label{lemma:exchange_condition}
Let the fused model be $q_{new,s,i} = w \cdot q_{s,i} + (1-w) \cdot q'_{s,i}$ for $w \in (0,1)$. On sample $s$, if the PM's original prediction is $i$ (i.e., $\arg\max(q_s) = i$), for the fused model's prediction for class $j$ to have a higher probability than for class $i$, it must satisfy:
$$w \cdot q_{s,j} + (1-w) \cdot q'_{s,j} > w \cdot q_{s,i} + (1-w) \cdot q'_{s,i}$$

For $w \in (0,1)$, this inequality is equivalent to:
$$\frac{q_{s,i} - q_{s,j}}{q'_{s,j} - q'_{s,i}} < \frac{1-w}{w}$$

Clearly, we need SM's prediction for $j$ to be greater than its prediction for $i$:
$$q'_{s,j} - q'_{s,i} > 0$$
This is termed the \textbf{exchange condition}.

Furthermore, if $\arg\max(q_{new,s}) = j$, then the fused model's prediction changes from $i$ to $j$. This is termed the \textbf{strict exchange condition} if the exchange condition is also met.

We define the left side as the Exchange Threshold, denoted $ET(s, i \to j)$. A flip occurs only when $ET < \frac{1-w}{w}$.

This lemma implies that only when the exchange condition is met can the fused model's prediction potentially differ from the PM's prediction. Any case not satisfying this condition cannot change the prediction outcome. Moreover, only when the strict exchange condition is met can the change in prediction outcome be determined.
\end{lemma}

\begin{assumption}[Overlap of Prediction Margins]
\label{ass:margin_overlap}
Typically, a model exhibits greater discrimination for samples it predicts correctly. However, to analyze the limits of static integration, we assume situations exist where this stratification is violated.

Specifically, we assume there exists at least one pair of samples, a correctable sample $t \in T$ and a vulnerable sample $f \in F$, such that the difference between the probability of the Primary Model's incorrectly predicted class and the true class for $t$ is \textbf{greater than or equal to} the difference between the probability of the correctly predicted class and the second most likely class for $f$.

Formally, there exists at least one sample $t \in T$ such that for some sample $f \in F$:
$$ (q_{t,i_t} - q_{t,k_t}) \ge (q_{f,k_f} - \max_{j \neq k_f} q_{f,j}) $$
where $i_t = \arg\max(q_t)$ is the Primary Model's incorrect prediction for $t$, $k_t$ is the true class of $t$, and $k_f = \arg\max(q_f)$ is the correct prediction for $f$.

A similar condition holds for the Secondary Model $q'$. Specifically, for the vulnerable sample $f \in F$ and the correctable sample $t \in T$, we assume:
$$ (q'_{f,j_f} - q'_{f,k_f}) \ge (q'_{t,k_t} - \max_{j \neq k_t} q'_{t,j}) $$
where $j_f = \arg\max(q'_f)$ is the Secondary Model's incorrect prediction for $f$, $k_f$ is the true class of $f$, and $k_t = \arg\max(q'_t)$ is the correct prediction for $t$.

\end{assumption}

\begin{theorem}[Impossibility of Static Perfection]
\label{thm:impossibility}
Consider a static linear ensemble system defined by $q_{new} = w \cdot q + (1-w) \cdot q'$ with weight $w \in (0,1)$. Under Assumption \ref{ass:margin_overlap}, where the Primary Model's confidence in its errors is greater than or equal to its confidence in its correct predictions for certain samples (specifically, $\exists t \in T, f \in F$ such that $ET(t) \ge ET(f)$), there exists \textbf{no single static weight} $w$ that allows the ensemble to achieve Theoretical Optimality.
\end{theorem}

\begin{proof}
Theoretical Optimality requires the fused model $q_{new}$ to correctly classify all samples in $S_{opt} = T \cup F \cup N_{corr}$. Since $T \cup F \subset S_{opt}$, a strict necessary condition for Theoretical Optimality is satisfying the classification constraints for the subset $T \cup F$. Therefore, we can omit the discussion of $N_{corr}$ and focus entirely on the conflicting conditions between $T$ and $F$.

For any sample $t \in T$ (where the Primary Model is incorrect and the Secondary Model is correct), the ensemble must flip the prediction to the true class $k_t$. According to Lemma \ref{lemma:exchange_condition}, this requires the weight ratio $\tau = \frac{1-w}{w}$ to exceed the Exchange Threshold for correction:
\begin{equation}
\label{eq:cond_t}
\tau > \frac{q_{t, i_t} - q_{t, k_t}}{q'_{t, k_t} - q'_{t, i_t}} = ET(t)
\end{equation}

Conversely, for any sample $f \in F$ (where the Primary Model is correct and the Secondary Model is incorrect), the ensemble must preserve the true class $k_f$ and avoid flipping to the incorrect class $j_f$. According to Lemma \ref{lemma:exchange_condition}, this requires $\tau$ to remain below the Exchange Threshold for error induction:
\begin{equation}
\label{eq:cond_f}
\tau < \frac{q_{f, k_f} - q_{f, j_f}}{q'_{f, j_f} - q'_{f, k_f}} = ET(f)
\end{equation}

For a single static weight $w$ to exist, there must be a valid $\tau$ satisfying both inequalities (\ref{eq:cond_t}) and (\ref{eq:cond_f}) for all pairs of samples in $T$ and $F$. This implies the necessary condition:
$$ ET(t) < \tau < ET(f) \implies ET(t) < ET(f) $$

However, Assumption \ref{ass:margin_overlap} states that there exists at least one pair of samples $t \in T$ and $f \in F$ such that the margin stratification is violated, i.e., $ET(t) \ge ET(f)$. Substituting this into the necessary condition yields a contradiction:
$$ \tau > ET(t) \ge ET(f) > \tau $$
This implies $\tau > \tau$, which is mathematically impossible. Consequently, no such static weight $w$ exists that can simultaneously correct $t$ and preserve $f$, rendering Theoretical Optimality unreachable.
\end{proof}

\begin{assumption}[Preservation of Margin Inversion]
\label{ass:preservation}
Let $Q_{ens}$ be a static linear ensemble of a set of models $\{q_i\}_{i=1}^M$ with weights summing to 1. 
Since static weighting applies a uniform linear scaling to all output probabilities, the smoothing effect on the margins of incorrect and correct predictions is consistent. 

Therefore, we assume that if the constituent models satisfy the margin violation condition (Assumption \ref{ass:margin_overlap}), the aggregated ensemble $Q_{ens}$ retains this property. Specifically, for the ensemble model, there still exists a pair of samples $t \in T, f \in F$ such that the confidence in the error remains greater than or equal to the confidence in the correct prediction:
$$ (Q_{ens,t, i_t} - Q_{ens,t, k_t}) \ge (Q_{ens,f, k_f} - \max_{j \neq k_f} Q_{ens,f,j}) $$
\end{assumption}

\begin{corollary}[Impossibility for Multi-Model Ensembles]
\label{cor:multi_model_impossibility}
Consider an ensemble of $M$ models $\mathcal{Q} = \{q_1, \dots, q_M\}$, where each model satisfies Assumption \ref{ass:margin_overlap}. Under Assumption \ref{ass:preservation}, there exists \textbf{no static weight distribution} $\{w_1, \dots, w_M\}$ that allows the ensemble to achieve Theoretical Optimality.
\end{corollary}

\begin{proof}
We proceed by induction on the number of models $M$. 

\begin{itemize}
    \item \textbf{Base Case ($M=2$):} 
    This is a direct application of Theorem \ref{thm:impossibility}. When the margin violation condition holds ($ET(t) \ge ET(f)$), no static weight can simultaneously satisfy the conflicting requirements for correction and preservation.

    \item \textbf{Inductive Step:}  
    Assume the statement holds for $k$ models. Let $Q_k$ be the optimal static ensemble of the first $k$ models. By the inductive hypothesis, $Q_k$ fails to achieve optimality. Furthermore, by Assumption \ref{ass:preservation}, $Q_k$ retains the margin violation property.

    Now, consider adding the $(k+1)$-th model, $q_{k+1}$. The new ensemble $Q_{k+1}$ can be formulated as a pairwise integration between the existing ensemble and the new model:
    $$ Q_{k+1} = \alpha \cdot Q_k + (1-\alpha) \cdot q_{k+1} $$
    Since both $Q_k$ (by inheritance) and $q_{k+1}$ (by definition) satisfy the margin violation condition, this mathematically reduces to the Base Case ($M=2$). Consequently, no static scalar $\alpha$ exists that can render $Q_{k+1}$ theoretically optimal without inducing errors.
\end{itemize}

By induction, Theoretical Optimality is unreachable for any finite number of models $M \ge 2$.
\end{proof}

\begin{theorem}[Algebraic Monotonicity of Ensemble Potential]
\label{thm:algebraic_monotonicity}
Consider an ensemble of $M$ models with weights $\mathbf{w} = [w_1, \dots, w_M]$ ($w_i \ge 0$).
For any sample $s$, let the \textbf{Discriminative Margin} of the $i$-th model be defined as the probability difference between the true class $k_s$ and the most competitive incorrect class $j$:
$$ \delta_i(s) = q_{s, k_s}^{(i)} - q_{s, j}^{(i)} $$
The ensemble correctly classifies sample $s$ if and only if the weighted sum of margins is positive:
$$ \Delta_{ens}(s, \mathbf{w}) = \sum_{i=1}^M w_i \delta_i(s) > 0 $$

We define \textbf{Model Improvement} as the reduction of margin overlap (Assumption \ref{ass:margin_overlap}). Specifically:
\begin{itemize}
    \item For correctable samples ($t \in T$), the incorrect prediction confidence decreases or the true class confidence increases. This means $\delta_i(t)$ increases (becomes less negative or turns positive).
    \item For preservation samples ($f \in F$), the correct prediction confidence increases or remains stable. This means $\delta_i(f)$ increases (becomes more positive).
\end{itemize}
Formally, if the system state updates such that $\delta'_i(s) \ge \delta_i(s)$ for all $i, s$, then the global maximum accuracy is strictly non-decreasing:
\begin{equation}
    \Delta ACC'_{max} \ge \Delta ACC_{max}
\end{equation}
\end{theorem}

\begin{proof}
Let $A(\mathbf{w})$ be the set of samples correctly classified by a specific weight vector $\mathbf{w}$.
$$ A(\mathbf{w}) = \{ s \in S \mid \sum_{i=1}^M w_i \delta_i(s) > 0 \} $$

Consider the updated state where margins improve to $\delta'_i(s) \ge \delta_i(s)$.
Since all weights are non-negative ($w_i \ge 0$), the linearity of the ensemble summation guarantees:
$$ \sum_{i=1}^M w_i \delta'_i(s) \ge \sum_{i=1}^M w_i \delta_i(s) $$

This implies a strict inclusion of classification capability for any fixed $\mathbf{w}$:
\begin{enumerate}
    \item If a sample was originally correct ($\sum w_i \delta_i(s) > 0$), the new sum is even larger, so it remains correct.
    \item If a sample was originally incorrect ($\sum w_i \delta_i(s) \le 0$), the increase may push the sum above 0, correcting it.
\end{enumerate}

Thus, for any candidate weight vector $\mathbf{w}$, the number of correct samples is non-decreasing: $|A'(\mathbf{w})| \ge |A(\mathbf{w})|$.
Let $\mathbf{w}^*$ be the optimal weight configuration for the original system. It follows that:
$$ \Delta ACC'_{max} = \max_{\mathbf{w}} \frac{|A'(\mathbf{w})|}{|S|} \ge \frac{|A'(\mathbf{w}^*)|}{|S|} \ge \frac{|A(\mathbf{w}^*)|}{|S|} = \Delta ACC_{max} $$
This proves that reducing the margin overlap (improving separability) strictly preserves or improves the global performance potential.
\end{proof}

These proofs derive the core intuition behind BPE. Theorem \ref{thm:impossibility} and Corollary \ref{cor:multi_model_impossibility} demonstrate the unavoidable theoretical defects of static integration, rendering dynamic integration necessary. However, the prevalence of static ensembles—despite their lack of theoretical advantage—raises a critical question: what form of dynamic weighting should be employed?

Theorem \ref{thm:algebraic_monotonicity} answer this by identifying that the behavioral adjustment of a model's internal output probabilities is the decisive factor determining the magnitude of accuracy (ACC) improvement. This insight guides a shift in the ensemble learning perspective from inter-model comparison to intra-model behavioral adjustment. Assuming that specific internal behavioral patterns form the foundation of effective ensembling (e.g., on average, a model's output probabilities tend to be more extreme when making correct predictions), appropriately amplifying these patterns can further enhance the capability of the ensemble model.

Beyond defining the theoretical limit of ACC improvement, appropriate adjustment of internal behavior facilitates a larger segmentation space, making the search for ensemble weights significantly easier. This expanded weight determination space addresses the limitations of traditional dynamic ensemble methods (such as DES/DCS), which often rely heavily on validation sets and lack generalization, leading to distorted weight determination. 

Taking the two-model case as an example, a larger difference $ET(f) - ET(t)$ implies a wider feasible range for the weight ratio $\tau$, thereby simplifying the determination of valid weights.

Furthermore, the analysis of internal differences aligns with the established phenomenon that distinct models exhibit inherent behavioral disparities, often rendering their raw output probabilities directly incomparable.

Collectively, these studies provide the theoretical intuition and basis for the BPE framework to define the ensemble perspective as the behavioral gap of a model relative to its own inherent behavioral profile.

\section{BPE: Behavioral Profiling Ensemble}\label{4}

To address the limitations of conventional DES/DCS methods regarding storage overhead, retrieval latency, and failure in high-dimensional spaces, this section proposes BPE, a validation-free dynamic ensemble framework. The BPE framework abandons the traditional evaluation paradigm that relies on external reference sets for neighborhood retrieval. Instead, it adopts the concept of constructing model \textit{behavioral profiles} via training set perturbations to reveal the intrinsic behavioral characteristics of base classifiers. Specifically, Section 3.1 elucidates the design philosophy of BPE, focusing on the paradigm shift from ``resume screening'' to ``behavioral profiling.'' Section 3.2 details a specific implementation algorithm, BPE-Entropy, based on information entropy. Finally, Section 3.3 provides a systematic analysis of the algorithmic complexity, demonstrating its advantages in computational efficiency and deployment flexibility.

\subsection{Design Philosophy of BPE}

The core design philosophy of the BPE framework lies in shifting the evaluation perspective of dynamic ensembles from \textit{external validation} to \textit{internal consistency}. Existing DES/DCS methods essentially follow a neighborhood-based validation paradigm, which assumes that the feature space surrounding a test sample shares similar decision boundaries with certain historical samples in the reference set. Under this assumption, the system infers a model's performance on the current sample by retrieving neighbors from the reference set and calculating the model's accuracy on these proxies.

However, this paradigm faces both theoretical and engineering challenges in practical applications. Theoretically, ``distance'' in high-dimensional feature spaces often loses its discriminative power \cite{hastie2009elements}, causing neighbors retrieved via Euclidean distance to be spurious (semantically irrelevant), thus introducing noisy evaluations. From an engineering perspective, reliance on a reference set necessitates maintaining an extensive reference set and performing expensive search computations during inference, which not only increases storage and latency overheads but also hinders the system's ability to handle streaming or privacy-sensitive data.

To this end, we propose a novel paradigm centered on model behavioral profiling. We posit that the decision reliability of a model for a given instance should not be determined solely by its historical performance metrics, but rather by its \textit{intrinsic state} when confronted with the current input. To capture these inherent behavioral attributes, we introduce training set perturbation as the core mechanism for establishing behavioral profiles. This is analogous to a stress test: rather than relying on a subject's past records, we observe their authentic reactions under controlled environmental pressure.

Within the BPE framework, we construct a testing environment for profiling by injecting noise into the training set. Through this process, the divergence patterns in the model's output distribution implicitly reveal its latent behavioral characteristics. Consequently, this allows us to discern when a model exhibits high uncertainty versus high predictive confidence. We allocate greater weight to a model when it demonstrates consistent confidence. By calibrating a model's response to different samples relative to its established behavioral profile, we derive a robust basis for dynamic weight assignment.

This approach offers several distinct advantages over existing methods:
\begin{itemize}
    \item \textbf{Minimal Storage Overhead:} The maintenance cost is significantly lower than that of traditional DES/DCS methods, as BPE requires storing only a compact behavioral profile (typically a vector) for each model, rather than a raw reference dataset.
    \item \textbf{Validation-Free \& Data Efficiency:} Since the construction of behavioral profiles does not mandate a separate validation set with ground-truth labels, the framework is highly effective in data-scarce regimes. It avoids the structural dependency on hold-out sets and eliminates the substantial computational costs associated with Out-Of-Fold (OOF) prediction techniques typically required to utilize full datasets.
    \item \textbf{Inference Efficiency:} Unlike retrieval-based DES/DCS methods, BPE necessitates only the calculation of lightweight statistical metrics during the inference phase. By eliminating complex nearest-neighbor search processes, BPE achieves a substantial advantage in inference speed, particularly in scenarios involving large-scale datasets.
\end{itemize}

Let $\mathcal{M}=\{h_1,\dots,h_K\}$ denote a pool of probabilistic base classifiers. BPE is a framework consisting of three components: (i) a \emph{behavioral profile definition}, which specifies the form of a model-specific profile $\mathcal{P}_k$ that summarizes the typical behavioral pattern of $h_k$; (ii) a \emph{profile construction rule}, which defines how $\mathcal{P}_k$ is obtained from the model responses under a prescribed profiling field; and (iii) a \emph{deviation scoring rule}, which computes a score to quantify how the behavior of $h_k$ on the current test instance deviates from $\mathcal{P}_k$. In Section 3.2 we present BPE-Entropy as one representative instantiation.

\begin{figure}[t]
    \centering
    \includegraphics[width=0.9\linewidth]{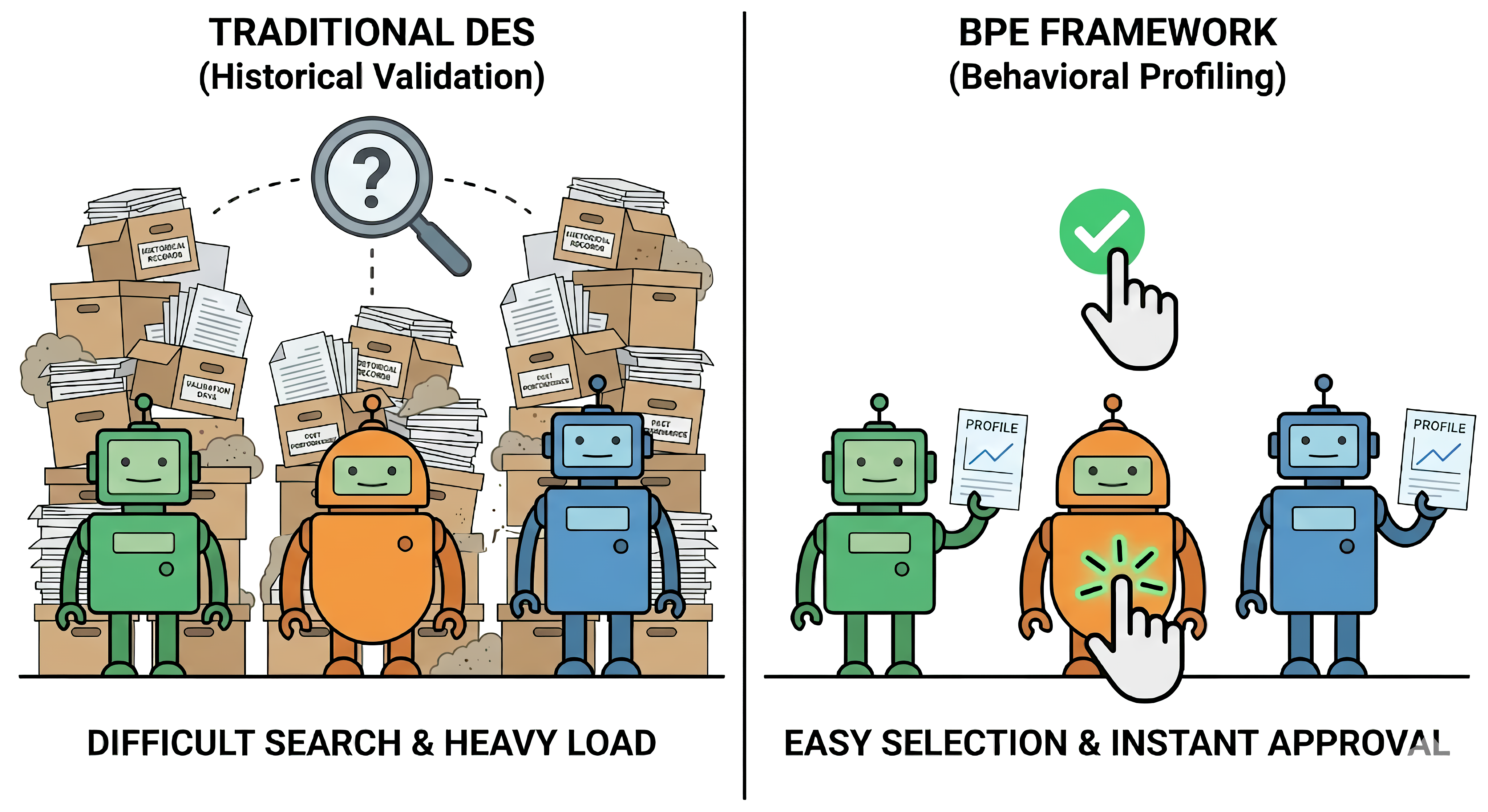}
    \caption{Comparison between BPE and traditional DES.}
    \label{bpe}
\end{figure}

\subsection{BPE-Entropy}

This section details the implementation of the BPE-Entropy algorithm. The algorithm utilizes the negative information entropy of the predictive distribution as the core metric for measuring intrinsic model uncertainty. By constructing behavioral profiles based on this metric, we facilitate dynamic weighting and achieve self-adaptive fusion of base classifiers.

Let the training set be denoted as $\mathcal{D} = \{(\mathbf{x}_i, y_i)\}_{i=1}^N$, and let $\mathcal{M} = \{h_1, \dots, h_K\}$ be a set of $K$ pre-trained heterogeneous base classifiers. For any input $\mathbf{x}$, the output of the $k$-th model regarding the $C$ classes is represented as a probability distribution vector $\mathbf{p}_k(\mathbf{x}) = [p_{k,1}, \dots, p_{k,C}]^\top$. To quantify the confidence of a model on a specific sample, we define the confidence score function $S(\cdot)$ as the negative information entropy of the predictive distribution:

\begin{equation}
S(\mathbf{p}_k(\mathbf{x})) = \sum_{c=1}^C p_{k,c}(\mathbf{x}) \log p_{k,c}(\mathbf{x})
\end{equation}

Note that $S(\mathbf{p})$ is the negative entropy of $\mathbf{p}$, i.e., $S(\mathbf{p}) = -H(\mathbf{p})$, so a larger $S$ value indicates a sharper predictive distribution and higher model certainty.

The first stage of the algorithm involves offline behavioral profiling, aimed at establishing a performance baseline for each model under simulated real-world conditions. Since profiling does not strictly rely on ground-truth labels, we move beyond traditional reference set retrieval and instead apply Gaussian perturbations to the training set feature space to simulate uncertainty in the testing environment. 

This strategic choice is underpinned by critical advantages regarding efficiency and stability. Primarily, utilizing the training set allows for full-data utilization without incurring the computational overhead associated with Out-of-Fold (OOF) predictions, ensuring that base models are trained on the complete dataset. Furthermore, the training set provides a substantially larger sample size than a typical validation set ($N_{train} \gg N_{val}$), which generally yields a more reliable estimation of average model behavior and reduces the risk of skewing the profile due to small-sample randomness. Crucially, relying on a validation set carries the risk of extreme scenarios—where hold-out data is inadvertently too difficult or too trivial—which would distort the estimated baseline; in contrast, applying fixed Gaussian perturbation creates a standardized stress-test environment that eliminates such bias caused by the arbitrary difficulty of unseen samples.

To implement this simulation while ensuring mathematical consistency, we explicitly handle nominal features prior to perturbation. Specifically, all categorical attributes are transformed into numerical vectors via One-Hot Encoding, ensuring the input $\mathbf{x}_i$ resides in a continuous metric space. The noise injection is applied to this encoded space. Although the perturbed one-hot vectors are no longer binary, this continuous relaxation is used only for probing model sensitivity during profiling and does not require semantic validity at the feature level. The selection of a Gaussian distribution for the noise vector $\epsilon$ is theoretically grounded in the Principle of Maximum Entropy proposed by Jaynes \cite{jaynes1957information}. According to this principle, under a fixed variance constraint, the Gaussian distribution has the maximum entropy among all distributions. Choosing this distribution ensures that we introduce the least amount of structural bias into the simulation, as it avoids implicit assumptions (such as the hard boundaries inherent in a uniform distribution) that are unjustified without prior knowledge. Consequently, for each sample $\mathbf{x}_i$, in the training set, we generate a perturbed sample $\tilde{\mathbf{x}}_i = \mathbf{x}_i + \epsilon$, where $\epsilon \sim \mathcal{N}(0, \delta^2 \mathbf{I})$ represents the unbiased noise vector.

Subsequently, we characterize the distribution of confidence scores for each base model across the entire perturbed dataset by calculating the mean $\mu_k$ and standard deviation $\sigma_k$:

\begin{equation}
\mu_k = \frac{1}{N} \sum_{i=1}^N S(\mathbf{p}_k(\tilde{\mathbf{x}}_i)), \quad \sigma_k = \sqrt{\frac{1}{N-1} \sum_{i=1}^N (S(\mathbf{p}_k(\tilde{\mathbf{x}}_i)) - \mu_k)^2}
\end{equation}

We define the tuple $\mathcal{P}_k = (\mu_k, \sigma_k)$ as the behavioral profile of model $h_k$, where $\mu_k$ represents the average confidence level when facing uncertainty, and $\sigma_k$ characterizes the sensitivity of its confidence fluctuations. This process is executed only once post-training, with the parameters $\mathcal{P}_k$ stored as lightweight metadata.

The second stage of the algorithm is online dynamic weighting. During the inference phase, for a new test sample $\mathbf{x}_{\text{test}}$, the system bypasses neighborhood searches and directly computes the instantaneous confidence for each model:

\begin{equation}
S_{\text{test},k} = S(\mathbf{p}_k(\mathbf{x}_{\text{test}}))
\end{equation}

To eliminate scale disparities in probabilities caused by differing algorithmic principles among base models, we introduce the concept of relative confidence. Utilizing the pre-stored profiling parameters, we normalize the instantaneous confidence via a Z-score transformation:

\begin{equation}
z_k(\mathbf{x}_{\text{test}}) = \frac{S_{\text{test},k} - \mu_k}{\sigma_k + \xi}
\end{equation}

where $\xi$ is a constant for numerical stability (we set $\xi = 10^{-12}$ in our experiments). The physical interpretation of $z_k$ is intuitive: it measures the number of standard deviations the current model confidence deviates from its historical behavioral profile. If $z_k > 0$, the model exhibits higher-than-usual confidence and will be assigned a larger weight; otherwise, its contribution will be down-weighted. In practice, we clip $z_k$ to the range $[-5, 5]$ to avoid extreme values when $\sigma_k$ is small. Finally, we map these normalized scores to non-negative weights. To amplify the contribution of high-confidence models, we employ an exponential mapping function followed by normalization:

\begin{equation}
w_k(\mathbf{x}_{\text{test}}) = \frac{\exp(\lambda \cdot z_k(\mathbf{x}_{\text{test}}))}{\sum_{j=1}^K \exp(\lambda \cdot z_j(\mathbf{x}_{\text{test}}))}
\end{equation}

where $\lambda$ is a sensitivity hyperparameter. Larger $\lambda$ yields a more peaked weight distribution. The final ensemble output $H(\mathbf{x}_{\text{test}})$ is derived from the weighted combination of the base models. Owing to the versatility of the BPE framework, these weights can be applied directly to probability vectors or to ranking results, flexibly adapting to different task requirements:

\begin{equation}
H(\mathbf{x}_{\text{test}}) = \sum_{k=1}^K w_k(\mathbf{x}_{\text{test}}) \cdot \Phi(\mathbf{x}_{\text{test}})_k
\end{equation}

where $\Phi(\cdot)_k$ represents the output format of the $k$-th model (e.g., probability vector or rank vector).

\begin{algorithm}[t]
\caption{The BPE-Entropy Algorithm}
\label{alg:bpe_entropy}
\begin{algorithmic}[1]
\Require 
    Training dataset $\mathcal{D} = \{(\mathbf{x}_i, y_i)\}_{i=1}^N$; 
    Pool of base classifiers $\mathcal{M} = \{h_1, \dots, h_K\}$; 
    Perturbation scale $\delta$; 
    Sensitivity factor $\lambda$;
    $\xi$ is a small constant $(10^{-12})$.
\Ensure 
    Ensemble prediction $H(\mathbf{x}_{\text{test}})$.

\Statex \textbf{// Phase 1: Offline Profiling}
\State Generate perturbed dataset $\tilde{\mathcal{D}}$ by adding Gaussian noise $\epsilon \sim \mathcal{N}(0, \delta^2\mathbf{I})$ to $\mathcal{D}$.
\For{$k = 1$ to $K$}
    \State Compute confidence scores (negative entropy) for all samples in $\tilde{\mathcal{D}}$ using $h_k$.
    \State Calculate mean $\mu_k$ and standard deviation $\sigma_k$ of the scores.
    \State Store behavioral profile $\mathcal{P}_k \leftarrow (\mu_k, \sigma_k)$.
\EndFor

\Statex \textbf{// Phase 2: Online Dynamic Weighting}
\State \textbf{Input:} New test sample $\mathbf{x}_{\text{test}}$.
\For{$k = 1$ to $K$}
    \State Compute instantaneous confidence: $S_{k} \leftarrow \text{NegEntropy}(h_k(\mathbf{x}_{\text{test}}))$.
    \State Compute Z-score standardization:
    \State \quad $z_k \leftarrow (S_{k} - \mu_k) / (\sigma_k + \xi)$
    \State Clip: $z_k \leftarrow \min\big(\max(z_k, -5), 5\big)$.
    \State Compute unnormalized weight (Exponential Mapping):
    \State \quad $w'_k \leftarrow \exp(\lambda \cdot z_k)$
\EndFor

\State Normalize weights: $w_k \leftarrow w'_k / \sum_{j=1}^K w'_j$.
\State \textbf{Return} Ensemble output $H(\mathbf{x}_{\text{test}}) = \sum_{k=1}^K w_k \cdot \Phi(\mathbf{x}_{\text{test}})_k$.
\end{algorithmic}
\end{algorithm}

\subsection{Complexity Analysis}

To comprehensively evaluate the deployment potential of the BPE framework, this section analyzes its complexity in terms of space complexity and time complexity, contrasting it with conventional neighborhood-based DES/DCS approaches such as KNORA \cite{ko2008dynamic}, LCA \cite{woods1997combination}, and RRC \cite{woloszynski2011probabilistic}.

Traditional DES/DCS methods necessitate maintaining the entire reference set during the inference phase for retrieval purposes. Assuming a reference set size of $N$, a feature dimensionality of $D$, and $K$ base models, the space complexity of traditional DES/DCS is primarily dictated by the scale of the reference set, approximating $\mathcal{O}(N \cdot D + N \cdot K)$. In large-scale data scenarios, memory consumption grows linearly with sample accumulation, imposing significant constraints on storage-limited devices. In contrast, BPE compresses the distributional characteristics into concise ``behavioral profiles'' during the offline phase. For each base model, BPE solely requires storing two scalar statistics ($\mu_k$ and $\sigma_k$), resulting in an auxiliary storage overhead of merely $\mathcal{O}(K)$. Given that $K \ll N \cdot D$, BPE substantially lowers the barriers to deployment.

Regarding time complexity, it is essential to distinguish between offline profiling and online inference. In the offline phase, both traditional DES/DCS and BPE require a comprehensive inference pass over the dataset to compute statistical features. However, during the online inference phase, traditional DES/DCS must execute a $k$-nearest neighbor search for each test sample. Even with optimized indexing structures such as KD-trees, the search cost often degrades to near-linear time in high-dimensional spaces, approaching $\mathcal{O}(N \cdot D)$ in practice, which can significantly increase inference latency as the reference set grows.

Conversely, BPE completely eliminates the dependence on historical data during inference, requiring only the computation of the current sample's entropy and Z-score. The complexity of computing entropy is $\mathcal{O}(C)$ (where $C$ is the number of classes), and the complexity of normalization and weighting is $\mathcal{O}(K)$. Consequently, the total online selection complexity of BPE is $\mathcal{O}(K \cdot C)$. Importantly, this inference cost is independent of the scale of the training or reference data (i.e., independent of $N$), and thus remains constant with respect to dataset size.

\begin{table}[htbp]
\centering
\caption{Complexity comparison between traditional KNN-based DES/DCS and the proposed BPE framework. Note that $N$, $D$, $K$, and $C$ denote the dataset size, feature dimension, number of models, and number of classes, respectively.}
\label{tab:complexity}

\begin{tabular*}{\linewidth}{@{\extracolsep{\fill}} l cccc }
\toprule
\textbf{Method} & \textbf{\makecell{Storage Cost\\(Space)}} & \textbf{\makecell{Offline Preparation\\(Time)}} & \textbf{\makecell{Online Inference\\(Time)}} & \textbf{\makecell{Dependency on\\Dataset Size ($N$)}} \\ 
\midrule
DES/DCS & High: $\mathcal{O}(N \cdot D)$ & High: $\mathcal{O}(N \cdot K)$ & High: $\mathcal{O}(N \cdot D)$ & Linear \\
BPE & \textbf{Low: $\mathcal{O}(K)$} & High: $\mathcal{O}(N \cdot K)$ & \textbf{Fast: $\mathcal{O}(K \cdot C)$} & \textbf{None} \\ 
\bottomrule
\end{tabular*}
\end{table}

\section{Experimental framework}\label{5}

\subsection{Dataset}
We selected 42 real-world datasets from OpenML for experimentation\cite{vanschoren2013openml}. These datasets originate from diverse high-quality sources, including the UCI Machine Learning Repository and NASA, covering a broad spectrum of domains ranging from software engineering (software defect prediction) and healthcare (disease diagnosis and physiological monitoring) to natural science (high-energy physics, biological classification, and climate simulation), business finance (customer churn prediction and credit approval), as well as cybersecurity and information technology (spam and phishing website detection). Detailed specifications of these datasets are provided in Table \ref{tab:datasets}. The configuration of the base model pool remains consistent with that of the simulation experiments. Furthermore, we conducted homogeneous model ensemble experiments; although such scenarios typically manifest as encapsulated algorithms in practical deployment (e.g., Random Forest), our homogeneous ensemble comprises 40 individual Decision Tree models, where the selection protocol for integration remains identical to that of the heterogeneous ensemble.

\begin{table}[htbp]
\centering
\caption{Summary description of the 42 datasets used in the experiments.}
\label{tab:datasets}
\small 
\setlength{\tabcolsep}{4pt} 

\begin{tabular}{@{} c c l c c c c c @{}}
\toprule
Dataset & ID & Dataset Name & \#Ex. & \#Atts. & \#Num. & \#Nom. & \#Cl. \\
\midrule
D1 & 9 & autos & 205 & 26 & 15 & 11 & 6 \\
D2 & 10 & lymph & 148 & 19 & 3 & 16 & 4 \\
D3 & 27 & colic & 368 & 23 & 7 & 16 & 2 \\
D4 & 29 & credit-approval & 690 & 16 & 6 & 10 & 2 \\
D5 & 30 & page-blocks & 5473 & 11 & 10 & 1 & 5 \\
D6 & 42 & soybean & 683 & 36 & 0 & 36 & 19 \\
D7 & 55 & hepatitis & 155 & 20 & 6 & 14 & 2 \\
D8 & 59 & ionosphere & 351 & 35 & 34 & 1 & 2 \\
D9 & 182 & satimage & 6430 & 37 & 36 & 1 & 6 \\
D10 & 333 & monks-problems-1 & 556 & 7 & 0 & 7 & 2 \\
D11 & 337 & SPECTF & 349 & 45 & 44 & 1 & 2 \\
D12 & 1017 & arrhythmia & 452 & 280 & 206 & 74 & 2 \\
D13 & 1036 & sylva\_agnostic & 14395 & 217 & 216 & 1 & 2 \\
D14 & 1040 & sylva\_prior & 14395 & 109 & 108 & 1 & 2 \\
D15 & 1049 & pc4 & 1458 & 38 & 37 & 1 & 2 \\
D16 & 1050 & pc3 & 1563 & 38 & 37 & 1 & 2 \\
D17 & 1053 & jm1 & 10885 & 22 & 21 & 1 & 2 \\
D18 & 1063 & kc2 & 522 & 22 & 21 & 1 & 2 \\
D19 & 1067 & kc1 & 2109 & 22 & 21 & 1 & 2 \\
D20 & 1442 & MegaWatt1 & 253 & 38 & 37 & 1 & 2 \\
D21 & 1460 & banana & 5300 & 3 & 2 & 1 & 2 \\
D22 & 1461 & bank-marketing & 45211 & 17 & 7 & 10 & 2 \\
D23 & 1463 & blogger & 100 & 6 & 0 & 6 & 2 \\
D24 & 1465 & breast-tissue & 106 & 10 & 9 & 1 & 6 \\
D25 & 1467 & climate-model-sim-crashes & 540 & 21 & 20 & 1 & 2 \\
D26 & 1480 & ilpd & 583 & 11 & 9 & 2 & 2 \\
D27 & 1487 & ozone-level-8hr & 2534 & 73 & 72 & 1 & 2 \\
D28 & 1489 & phoneme & 5404 & 6 & 5 & 1 & 2 \\
D29 & 1494 & qsar-biodeg & 1055 & 42 & 41 & 1 & 2 \\
D30 & 1496 & ringnorm & 7400 & 21 & 20 & 1 & 2 \\
D31 & 1511 & wholesale-customers & 440 & 9 & 7 & 2 & 2 \\
D32 & 4538 & GesturePhaseSeg-Processed & 9873 & 33 & 32 & 1 & 5 \\
D33 & 6332 & cylinder-bands & 540 & 40 & 18 & 22 & 2 \\
D34 & 40670 & dna & 3186 & 181 & 0 & 181 & 3 \\
D35 & 40685 & shuttle & 58000 & 10 & 9 & 1 & 7 \\
D36 & 40693 & xd6 & 973 & 10 & 0 & 10 & 2 \\
D37 & 40705 & tokyo1 & 959 & 45 & 42 & 3 & 2 \\
D38 & 40945 & Titanic & 1309 & 14 & 6 & 3 & 2 \\
D39 & 40966 & MiceProtein & 1080 & 82 & 77 & 5 & 8 \\
D40 & 40981 & Australian & 690 & 15 & 6 & 9 & 2 \\
D41 & 41144 & madeline & 3140 & 260 & 259 & 1 & 2 \\
D42 & 41981 & SAT11-HAND-runtime-class & 296 & 116 & 115 & 1 & 14 \\
\bottomrule
\addlinespace 
\multicolumn{8}{p{\linewidth}}{** For datasets with more than 10,000 samples, considering computational resource limitations, we uniformly downsample to 10,000 samples.  Since we repeat the experiment 50 times with different random seeds, most samples from each dataset will actually be selected.} \\
\end{tabular}
\end{table}

\subsection{Base learner Pool}
We assembled a diverse pool containing 13 heterogeneous base classifiers, spanning the mainstream algorithmic paradigms in machine learning:
\begin{itemize}
    \item \textbf{Bagging-based Models:} Including Random Forest \cite{breiman2001random} and Extra Trees \cite{geurts2006extremely}. These methods mitigate the variance of individual decision trees through random feature subspaces and bootstrap aggregation.
    \item \textbf{Boosting-based Models:} Including XGBoost \cite{chen2016xgboost}, LightGBM \cite{ke2017lightgbm}, CatBoost \cite{prokhorenkova2018catboost}, and AdaBoost \cite{freund1997decision}. These models excel in modeling complex non-linear boundaries by iteratively optimizing residuals.
    \item \textbf{Linear and Discriminant Models:} Including Logistic Regression \cite{cox1958regression} and Linear Discriminant Analysis (LDA) \cite{fisher1936use}, which provide robust linear approximations to prevent overfitting.
    \item \textbf{Non-linear and Distance-based Models:} Including MLPClassifier \cite{rumelhart1986learning}, SVC (with probability estimation) \cite{cortes1995support}, and k-Nearest Neighbors (KNN) \cite{cover1967nearest}.
    \item \textbf{Probabilistic Models:} Gaussian Naive Bayes \cite{zhang2004optimality} is included to provide a probabilistic baseline.
\end{itemize}

In practical applications, an ensemble system typically does not use exactly the same model pool for all datasets. Prior work has also pointed out that overly weak models may harm ensemble performance \cite{caruana2004ensemble,martinezmunoz2006pruning,martinezmunoz2007boosting,zhou2002many}. Consistent with this insight, our experiments show that when no screening is performed and obviously mismatched models are forcibly included, the classification accuracy of all evaluated algorithms decreases. Therefore, it is common in practice to construct an appropriate candidate model pool by combining prior knowledge with a screening strategy; for example, in classical machine learning, candidate algorithms are selected according to whether the data satisfy certain model assumptions, and in data science competitions, models with clearly low accuracy are often excluded based on public leaderboards or cross-validation performance.

To simulate this practical workflow, we introduce a simple performance-based screening mechanism as a unified experimental protocol for model-pool construction. Concretely, for each dataset we (i) split the training set into 80\%/20\% for screening, (ii) compute the screening accuracy $\mathrm{ACC}(h_k)$ of each base learner on the 20\% split, (iii) retain only the models satisfying $\mathrm{ACC}(h_k) \ge \mathrm{ACC}_{\text{best}} \times (1 - \alpha)$ with $\alpha = 0.15$, and (iv) retrain the retained models on the full training set for subsequent evaluation. Besides mimicking real-world workflows, this dataset-dependent base model pool introduces diversity across datasets, which makes the empirical comparison more robust. Meanwhile, for each dataset, all ensemble methods are evaluated using the same screened model pool, ensuring a fair comparison between algorithms. We emphasize that this screening mechanism is part of the experimental protocol for model-pool construction; the BPE framework itself does not rely on a validation set for instance-wise weighting.

To guarantee the fairness of the experiments, we adopted standard parameter settings for all base learners, deliberately refraining from specific hyperparameter tuning. This protocol ensures that the observed performance gains are attributable solely to the efficacy of the ensemble integration rather than individual model optimization. The specific hyperparameter configurations for all base learners are detailed in Table \ref{tab:hyperparameters}.

\subsection{Baseline Methods}
To rigorously evaluate the performance of BPE, we compare it against a comprehensive set of competitive ensemble strategies, categorized into static and dynamic approaches.

Our experiments involve multiple conventional DES/DCS baselines that require a reference set. We observe that splitting an additional fixed reference set from the training data significantly reduces the data available for training base models, which in turn leads to a clear drop in classification accuracy for these reference-dependent methods. To ensure these methods can utilize the full dataset and achieve accuracy that is comparable with other methods, we adopt their most common standard implementation in practice: using Out-of-Fold (OOF) predictions to fully exploit all training data \cite{wolpert1992stacked}. Concretely, we perform $5$-fold OOF, where base models are trained on each fold and used to generate OOF predictions on the held-out fold; after OOF predictions are obtained, each base model is retrained on the full training set. Note that $5$-fold OOF introduces additional training cost, and therefore creates a significant computational gap between these reference-dependent methods and BPE.

In addition, to present the impact of different protocols, we report in the Appendix the results of a standard engineering alternative: ``fixed train/reference split + retraining on the full training set after training.'' This alternative typically costs only about $2\times$ the computation of BPE, which is smaller than $5$-fold OOF, but its accuracy is generally worse than the OOF-based strategy.

\paragraph{Static Ensemble Baselines}
\begin{itemize}
    \item \textbf{Single Best Classifier:} (abbreviated as \textbf{SB} in the following tables) This oracle-style reference baseline selects the single best-performing base model according to the reference accuracy used in the screening protocol, and then evaluates this selected model on the test set.
    \item \textbf{Simple Average:} (abbreviated as \textbf{SA} in the following tables) This method operates on the assumption that all base models contribute equally. It computes the arithmetic mean of the output probability vectors from all base classifiers as the final prediction. Despite its simplicity, consistent with Dietterich's statistical perspective, simple averaging effectively smooths decision boundaries and reduces variance by aggregating multiple independent hypotheses, serving as a robust baseline that is often difficult to outperform significantly \cite{clemen1989combining, timmermann2006forecast}.
    \item \textbf{Weighted Average:} (abbreviated as \textbf{WA} in the following tables) An empirical strategy derived from engineering practice. In industrial applications and data science competitions, the Out-of-Fold (OOF) accuracy generated during the cross-validation phase is typically utilized as static weights. This method relies on the intuitive assumption that models demonstrating superior performance on the reference set should be assigned higher weights. during the testing phase.
    \item \textbf{Median Average:} (abbreviated as \textbf{MA} in the following tables) In scenarios where output probability vectors exhibit extreme outliers, the median offers superior robustness compared to the mean and serves as a commonly utilized ensemble technique.
\end{itemize}

\paragraph{Dynamic Ensemble Baselines}
\begin{itemize}
    \item \textbf{LCA (Local Class Accuracy):} A classic dynamic selection method proposed by Woods et al. \cite{woods1997combination}. It assigns weights based on local class accuracy. For each test sample, LCA retrieves the $k$-nearest neighbors from the reference set and calculates the classification accuracy of each base model regarding the predicted class within this local neighborhood. If a model demonstrates high historical local accuracy, it is deemed an ``expert'' for that specific sample.
    \item \textbf{KNORA-Union (k-Nearest Oracles Union):} Proposed by Ko et al. \cite{ko2008dynamic}, this algorithm is predicated on the ``Nearest Oracle'' assumption. Unlike LCA, it employs a direct voting mechanism: for each neighbor of the test sample, every base model that correctly classifies that neighbor is granted a vote. This relatively lenient strategy tends to involve a broader subset of models in the decision-making process, thereby reducing predictive variance.
    \item \textbf{KNORA-Eliminate (k-Nearest Oracles Eliminate):} A stricter variant within the KNORA family. It mandates that a base model must correctly classify \textit{all} $k$ neighbors surrounding the test sample to be eligible for voting. If no model satisfies this unanimous condition, the value of $k$ is iteratively decremented until a qualifying subset is identified.
    \item \textbf{RRC (Randomized Reference Classifier):} Proposed by Woloszynski \cite{woloszynski2011probabilistic} to mitigate the hard-thresholding limitations of traditional $k$-NN methods. RRC introduces a continuous probabilistic evaluation perspective by constructing a randomized reference classifier framework. It utilizes Gaussian kernels as potential functions to assign varying weights based on the Euclidean distance between reference and test samples, effectively smoothing the evaluation of local competence.
    \item \textbf{MCB (Multiple Classifier Behavior):} Proposed by Giacinto et al. \cite{giacinto2001design}, this method performs selection by analyzing the behavioral consistency of base models. Unlike methods focused solely on accuracy, MCB retrieves neighbors in the reference set that exhibit similar ``output profiles'' (vectors composed of all base model predictions) to the current test sample. The premise is that if historical output patterns match the current pattern, the labels of those neighbors should guide the final decision.
    \item \textbf{DES-AS:} A state-of-the-art synergy-based framework proposed by Zhang et al. \cite{zhang2025des}. It utilizes Algorithm Shapley, a game-theoretic variant of the Shapley Value, to measure the synergy effect among classifiers. By calculating the marginal contribution of each model to all possible ensemble permutations, it selects classifiers with positive synergy values for dynamic aggregation.
    \item \textbf{DES-KNN:} This family of methods selects an ensemble for each test instance by jointly considering local accuracy and diversity within its Region of Competence (RoC) \cite{desouto2008dcs}. Concretely, base classifiers are ranked by decreasing accuracy and (separately) by increasing redundancy/diversity estimated on the RoC. The final Ensemble of Competence (EoC) is formed by combining the top-$p_a$ most accurate models with the top-$p_b$ most diverse models (with duplicates removed). In our implementation, the diversity score of a classifier is obtained from its average pairwise diversity to the remaining classifiers on the RoC, using common measures such as Double-Fault (DF), the $Q$-statistic, and the Ratio of Errors (RE).
\end{itemize}

\begin{table}[htbp]
\centering
\caption{Hyperparameter configurations for the base model pool.}
\label{tab:hyperparameters}
\begin{tabular*}{\columnwidth}{@{\extracolsep{\fill}} lll @{}}
\toprule
\textbf{Category} & \textbf{Algorithm} & \textbf{Specific Configurations / Key Parameters} \\ 
\midrule
\multirow{2}{*}{\textbf{Bagging}} & Random Forest & $n\_estimators=100$, $criterion='gini'$ \\
 & Extra Trees & $n\_estimators=100$, $bootstrap=False$ \\ 
\midrule
\multirow{5}{*}{\textbf{Boosting}} & XGBoost & $eval\_metric='logloss'$, $use\_label\_encoder=False$ \\
 & LightGBM & $verbose=-1$ \\
 & CatBoost & $logging\_level='Silent'$, $thread\_count=1$ \\
 & Gradient Boosting & $n\_estimators=100$, $learning\_rate=0.1$ \\
 & AdaBoost & $n\_estimators=50$, $algorithm='SAMME.R'$ \\ 
\midrule
\multirow{2}{*}{\textbf{Linear}} & Logistic Regression & $max\_iter=1000$, $solver='lbfgs'$ \\
 & LDA & Default (SVD solver) \\ 
\midrule
\multirow{4}{*}{\textbf{Others}} & MLP Classifier & $max\_iter=3000$, $early\_stopping=True$ \\
 & KNN & $n\_neighbors=5$, $weights='uniform'$ \\
 & SVC & $probability=True$, $kernel='rbf'$ \\
 & Gaussian NB & Default \\ 
\bottomrule
\end{tabular*}
\end{table}

\section{Experimental Results}\label{6}
In this section, we present the evaluation of the proposed BPE framework. First, Section 6.1 analyzes the experimental results in the heterogeneous ensemble scenario, followed by the evaluation of the homogeneous ensemble setting in Section 6.2. Subsequently, Section 6.3 investigates the hyperparameter sensitivity of the algorithm, focusing primarily on the heterogeneous ensemble. This focus is motivated by two key factors: (1) heterogeneous ensembles generally yield superior accuracy, a phenomenon consistently observed both in our experiments and in the broader ensemble learning literature; and (2) dynamic ensemble algorithms inherently introduce additional computational and storage overheads. This intrinsic cost conflicts with the primary advantages of homogeneous ensembles, which are typically favored for their computational efficiency and ease of deployment. Consequently, the homogeneous setting is included in our experiments primarily as an auxiliary analysis.

\subsection{Heterogeneous ensemble}
In the heterogeneous setting, we construct the model pool by combining diverse learning algorithms with the hyperparameter configurations summarized in Table\ref{tab:hyperparameters}.

The proposed BPE framework demonstrates consistent superiority over competitive static and dynamic baselines. As detailed in Table \ref{tab:heterogeneous_acc}, BPE achieves the highest average classification accuracy of 87.17\%, outperforming the best single classifier (86.78\%) and the strongest baseline ensemble, RRC (87.08\%). The overall advantage of BPE is further corroborated by the Friedman ranking analysis in Fig. \ref{fig:friedman_hetero}: BPE obtains the best (lowest) average rank of 2.167, substantially ahead of the runner-up RRC (4.310).

To validate the statistical significance of these improvements, Wilcoxon signed-rank tests were conducted at a significance level of $\alpha = 0.05$. For each dataset, we report the mean accuracy over 50 runs, and the tests were performed across the 42 datasets using these per-dataset means. The results, summarized in Table \ref{tab:heterogeneous_oof_wilcoxon}, show that the null hypothesis is rejected for all comparisons between BPE and the baselines (all $p < 0.05$), indicating that BPE provides statistically significant gains over both static fusion rules and neighborhood-based DES/DCS methods under the same OOF protocol.

\begin{table}[htbp]
\centering
\caption{Average classification accuracy obtained by BPE and baseline models over 42 datasets in \textbf{heterogeneous (OOF)} ensemble case. The best results are highlighted in \textbf{boldface}.}
\label{tab:heterogeneous_acc}
\setlength{\tabcolsep}{2pt}
\resizebox{\textwidth}{!}{
\begin{tabular}{l cccc cc ccccccc c}
\toprule
Dataset & \multicolumn{4}{c}{Static} & \multicolumn{2}{c}{DCS} & \multicolumn{7}{c}{DES} & \multicolumn{1}{c}{BPE} \\
\cmidrule(lr){2-5} \cmidrule(lr){6-7} \cmidrule(lr){8-14} \cmidrule(lr){15-15}
& SB & SA & MA & WA & LCA & MCB & DES-AS & DF & Q & RE & KNE & KNU & RRC & Entropy \\
\midrule
D1 & 78.98 & 78.92 & \textbf{79.38} & 78.98 & 79.34 & \textbf{79.38} & 78.98 & 79.15 & 79.08 & 78.33 & 79.18 & 79.25 & \textbf{79.38} & 79.28 \\
D2 & 84.88 & 85.58 & 85.67 & 86.00 & \textbf{86.19} & 85.63 & 85.72 & 85.30 & 85.44 & 84.93 & 85.53 & 85.77 & 86.09 & 85.86 \\
D3 & 82.81 & 84.07 & 83.98 & \textbf{84.09} & 83.95 & 83.77 & 83.86 & 83.84 & 83.73 & 83.48 & 84.00 & 83.86 & 83.93 & 84.07 \\
D4 & 86.99 & 87.26 & 87.29 & 87.30 & 87.34 & 87.19 & \textbf{87.37} & 87.17 & 87.18 & 86.86 & 87.23 & \textbf{87.37} & \textbf{87.37} & 87.32 \\
D5 & 97.13 & 97.29 & \textbf{97.32} & 97.29 & 97.29 & 97.26 & 97.30 & 97.24 & 97.24 & 97.23 & 97.26 & 97.29 & 97.30 & 97.30 \\
D6 & 93.20 & 93.20 & 93.22 & 93.19 & 93.23 & 93.30 & 93.25 & 93.15 & 93.25 & 93.00 & 93.29 & 93.24 & 93.22 & \textbf{93.40} \\
D7 & 82.85 & 83.36 & 83.45 & 83.36 & 83.45 & 83.02 & 83.32 & 83.49 & 83.45 & 83.49 & 83.23 & 83.36 & 83.49 & \textbf{83.74} \\
D8 & \textbf{93.19} & 92.81 & 92.91 & 92.83 & 92.89 & 92.96 & 93.00 & 92.83 & 92.81 & 92.45 & 92.94 & 92.96 & 92.91 & 93.13 \\
D9 & 91.90 & 92.04 & 91.95 & 92.04 & 92.08 & 92.05 & 92.06 & 91.97 & 91.99 & 91.91 & 92.02 & \textbf{92.09} & 92.08 & 92.07 \\
D10 & 98.99 & 98.99 & 98.99 & 98.99 & 98.99 & 98.99 & 98.99 & 98.99 & 98.99 & 98.98 & 98.99 & \textbf{99.01} & 98.99 & 98.99 \\
D11 & 87.92 & 88.10 & 87.79 & 88.11 & 88.02 & 87.45 & 87.64 & 87.56 & 87.37 & 87.05 & 88.06 & 87.83 & 88.02 & \textbf{88.13} \\
D12 & 80.87 & 80.65 & \textbf{80.91} & 80.63 & 80.49 & 80.26 & 80.41 & 80.31 & 80.26 & 79.99 & 80.65 & 80.38 & 80.51 & 80.65 \\
D13 & 99.35 & 99.42 & 99.42 & 99.42 & 99.42 & 99.41 & 99.41 & 99.40 & 99.40 & 99.34 & 99.40 & 99.42 & 99.42 & \textbf{99.43} \\
D14 & 99.39 & \textbf{99.46} & 99.45 & \textbf{99.46} & \textbf{99.46} & 99.44 & 99.45 & 99.43 & 99.43 & 99.38 & 99.44 & \textbf{99.46} & \textbf{99.46} & \textbf{99.46} \\
D15 & 90.65 & 91.18 & 91.20 & 91.19 & 91.23 & 91.09 & \textbf{91.24} & 90.94 & 90.95 & 90.79 & 91.08 & 91.18 & 91.23 & 91.19 \\
D16 & 89.70 & 89.99 & 89.91 & 89.98 & 89.99 & 89.85 & 89.83 & 89.84 & 89.88 & 89.86 & 89.93 & 89.88 & 89.98 & \textbf{90.06} \\
D17 & 81.24 & 81.49 & 81.45 & 81.49 & 81.66 & 81.64 & 81.67 & 81.70 & 81.63 & 81.68 & 81.63 & 81.73 & 81.73 & \textbf{81.89} \\
D18 & 83.38 & 83.63 & \textbf{83.77} & 83.62 & 83.62 & 83.44 & \textbf{83.77} & 83.53 & 83.46 & 83.50 & 83.64 & 83.59 & 83.68 & 83.64 \\
D19 & 85.54 & 86.47 & 86.30 & 86.47 & 86.45 & 86.28 & 86.43 & 86.31 & 86.38 & 86.36 & 86.50 & 86.48 & \textbf{86.54} & 86.53 \\
D20 & 88.84 & 89.50 & 89.50 & 89.50 & \textbf{89.53} & 89.21 & 89.39 & 89.26 & 88.92 & 88.61 & 89.37 & 89.50 & 89.45 & 89.50 \\
D21 & \textbf{90.23} & 90.06 & 90.04 & 90.06 & 90.01 & 89.95 & 89.98 & 89.86 & 89.90 & 89.85 & 90.03 & 90.00 & 90.01 & 90.13 \\
D22 & 90.25 & 90.36 & \textbf{90.37} & 90.36 & 90.34 & 90.26 & 90.35 & 90.32 & 90.30 & 90.33 & 90.30 & 90.34 & 90.34 & \textbf{90.37} \\
D23 & 81.87 & 81.60 & 81.73 & 81.53 & 81.33 & 80.80 & 81.33 & 81.20 & 81.07 & 80.73 & 81.60 & 81.00 & 81.53 & \textbf{82.00} \\
D24 & 40.25 & 42.13 & 41.81 & 42.50 & 42.50 & 42.31 & 42.00 & 41.75 & 42.06 & 41.94 & 42.06 & 42.38 & \textbf{42.63} & 42.31 \\
D25 & 91.01 & 91.01 & 91.02 & 91.01 & 91.00 & 90.93 & 90.95 & 90.95 & 90.99 & 91.09 & 90.99 & 91.02 & 91.00 & \textbf{91.10} \\
D26 & 70.82 & 71.05 & 70.96 & 71.09 & 71.31 & 71.01 & 70.96 & 70.67 & 70.72 & 70.41 & 70.93 & 71.29 & 71.29 & \textbf{71.39} \\
D27 & 94.32 & 94.50 & 94.49 & 94.50 & 94.52 & 94.50 & \textbf{94.53} & \textbf{94.53} & \textbf{94.53} & 94.50 & 94.49 & 94.51 & 94.52 & 94.52 \\
D28 & \textbf{90.64} & 90.45 & 90.46 & 90.46 & 90.45 & 90.44 & 90.43 & 90.39 & 90.40 & 90.16 & 90.47 & 90.48 & 90.45 & 90.58 \\
D29 & 87.04 & 87.59 & 87.41 & 87.56 & 87.55 & 87.39 & \textbf{87.62} & 87.46 & 87.50 & 87.55 & 87.53 & 87.59 & 87.56 & \textbf{87.62} \\
D30 & 97.78 & \textbf{97.82} & 97.79 & \textbf{97.82} & \textbf{97.82} & 97.80 & 97.81 & 97.77 & 97.78 & 97.62 & 97.80 & \textbf{97.82} & \textbf{97.82} & \textbf{97.82} \\
D31 & 90.76 & 91.55 & 91.42 & 91.50 & 91.52 & 91.23 & 91.33 & 91.08 & 91.09 & 91.09 & 91.33 & 91.41 & 91.48 & \textbf{91.59} \\
D32 & 67.84 & 68.07 & 68.08 & 68.08 & 68.19 & 68.34 & 68.29 & 67.92 & 67.94 & 67.84 & 68.09 & 68.30 & 68.20 & \textbf{68.40} \\
D33 & 79.81 & 80.12 & 80.07 & 80.14 & 80.28 & 80.14 & 80.32 & 80.12 & 80.16 & 79.94 & 80.28 & 80.36 & 80.30 & \textbf{80.70} \\
D34 & 96.29 & \textbf{96.42} & 96.39 & \textbf{96.42} & \textbf{96.42} & 96.39 & \textbf{96.42} & 96.32 & 96.33 & 96.16 & 96.35 & 96.41 & \textbf{96.42} & \textbf{96.42} \\
D35 & 99.83 & 99.21 & 99.17 & \textbf{99.85} & 99.25 & 99.25 & 99.82 & 99.20 & 99.20 & 99.19 & 99.25 & 99.25 & 99.25 & \textbf{99.85} \\
D36 & \textbf{100.00} & \textbf{100.00} & \textbf{100.00} & \textbf{100.00} & \textbf{100.00} & \textbf{100.00} & \textbf{100.00} & \textbf{100.00} & \textbf{100.00} & \textbf{100.00} & \textbf{100.00} & \textbf{100.00} & \textbf{100.00} & \textbf{100.00} \\
D37 & 92.84 & 92.80 & 92.78 & 92.79 & 92.73 & 92.68 & 92.74 & 92.54 & 92.58 & 92.41 & 92.74 & 92.73 & 92.74 & \textbf{92.92} \\
D38 & 80.65 & 81.20 & 81.21 & 81.19 & 81.17 & 81.00 & 81.03 & 80.69 & 80.67 & 80.50 & 81.10 & 81.14 & 81.17 & \textbf{81.29} \\
D39 & 99.43 & 99.59 & 99.59 & 99.60 & 99.62 & 99.62 & \textbf{99.67} & 99.55 & 99.54 & 99.40 & 99.58 & 99.62 & 99.62 & 99.62 \\
D40 & 85.86 & 86.58 & 86.42 & 86.58 & 86.46 & 86.46 & 86.54 & 86.41 & 86.40 & 86.15 & 86.60 & 86.47 & 86.46 & \textbf{86.76} \\
D41 & 81.59 & 82.21 & 82.17 & 82.21 & 82.23 & 81.89 & 82.23 & 81.98 & 81.99 & 81.19 & 81.99 & 82.22 & 82.23 & \textbf{82.26} \\
D42 & 57.68 & 57.52 & \textbf{58.14} & 57.57 & 57.64 & 57.52 & 57.55 & 57.07 & 56.73 & 56.80 & 57.55 & 57.73 & 57.61 & 57.89 \\
\midrule
Average & 86.78 & 87.03 & 87.03 & 87.07 & 87.07 & 86.94 & 87.02 & 86.89 & 86.87 & 86.72 & 87.01 & 87.04 & 87.08 & \textbf{87.17} \\
\bottomrule
\end{tabular}
}
\end{table}

\begin{table}[htbp]
\centering
\caption{Wilcoxon signed-rank test results for comparing BPE with thirteen baseline models in \textbf{heterogeneous (OOF)} ensemble case.}
\label{tab:heterogeneous_oof_wilcoxon}
\begin{tabular*}{\textwidth}{@{\extracolsep{\fill}} ll cccc }
\toprule
Ensemble type & Comparison & $R^+$ & $R^-$ & Hypothesis & $p$-value \\
\midrule
\multirow{4}{*}{Static} 
 & BPE vs. SB     & 784.0 & 36.0  & Rejected for BPE at 5\% & 0.0000** \\
 & BPE vs. SA   & 698.0 & 5.0   & Rejected for BPE at 5\% & 0.0000** \\
 & BPE vs. MA   & 681.0 & 99.0  & Rejected for BPE at 5\% & 0.0000** \\
 & BPE vs. WA & 706.0 & 74.0  & Rejected for BPE at 5\% & 0.0000** \\
\midrule
\multirow{2}{*}{DCS}    
 & BPE vs. LCA      & 631.0 & 110.0 & Rejected for BPE at 5\% & 0.0002** \\
 & BPE vs. MCB      & 731.0 & 10.0  & Rejected for BPE at 5\% & 0.0000** \\
\midrule
\multirow{7}{*}{DES}    
 & BPE vs. DES-AS   & 718.0 & 62.0  & Rejected for BPE at 5\% & 0.0000** \\
 & BPE vs. DF       & 819.0 & 1.0   & Rejected for BPE at 5\% & 0.0000** \\
 & BPE vs. Q        & 819.0 & 1.0   & Rejected for BPE at 5\% & 0.0000** \\
 & BPE vs. RE       & 861.0 & 0.0   & Rejected for BPE at 5\% & 0.0000** \\
 & BPE vs. KNE      & 741.0 & 0.0   & Rejected for BPE at 5\% & 0.0000** \\
 & BPE vs. KNU      & 699.0 & 42.0  & Rejected for BPE at 5\% & 0.0000** \\
 & BPE vs. RRC      & 610.0 & 131.0 & Rejected for BPE at 5\% & 0.0005** \\
\bottomrule
\addlinespace
\multicolumn{6}{l}{** Close to the $p$-value means that statistical differences are found with $\alpha = 0.05$.} \\
\end{tabular*}
\end{table}

\begin{figure}[t]
    \centering
    \includegraphics[width=0.9\linewidth]{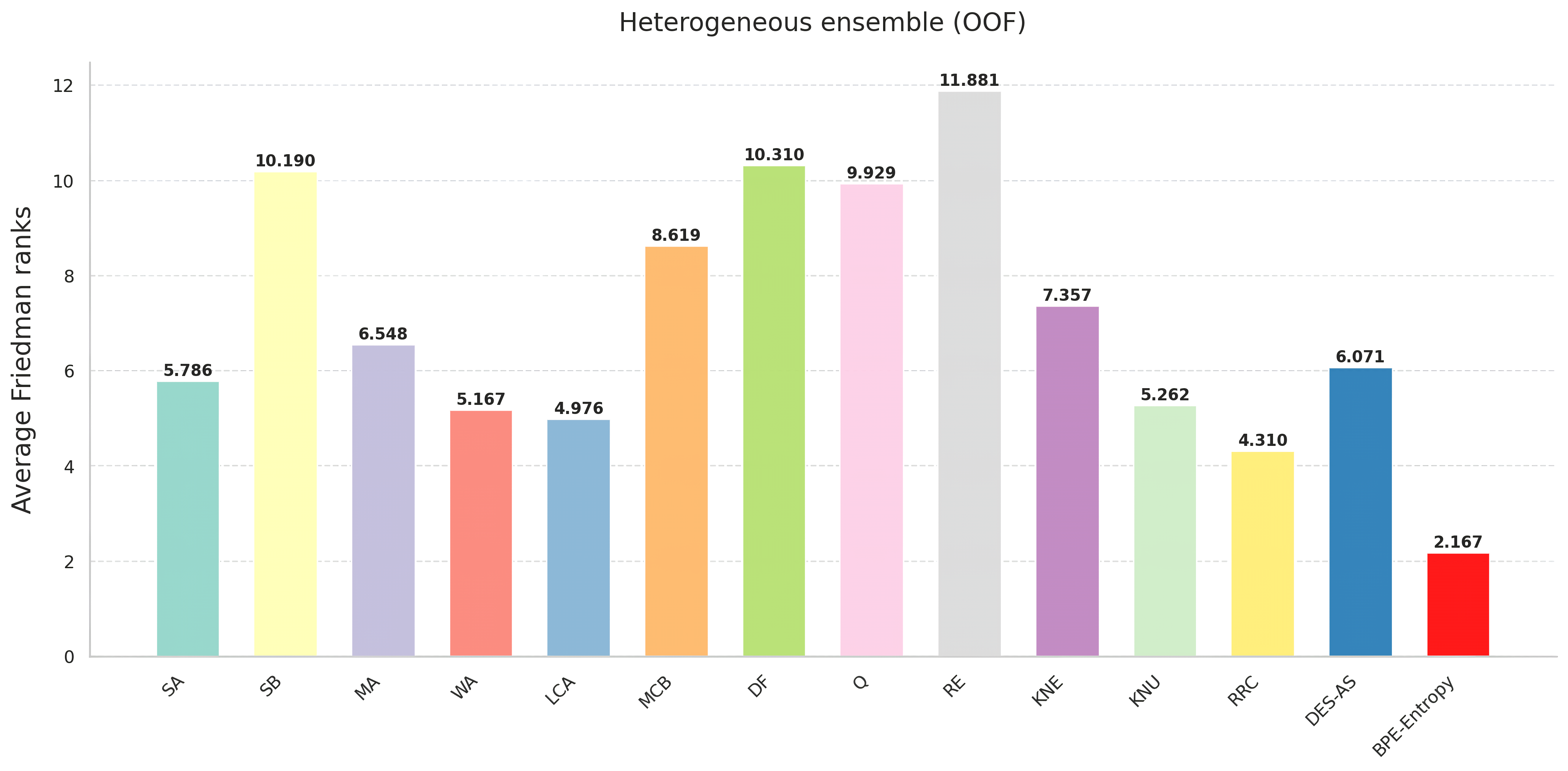}
    \caption{Average Friedman ranks of BPE and thirteen baseline models in \textbf{heterogeneous} ensemble case.}
    \label{fig:friedman_hetero}
\end{figure}

\begin{figure}[t]
    \centering
    \includegraphics[width=0.9\linewidth]{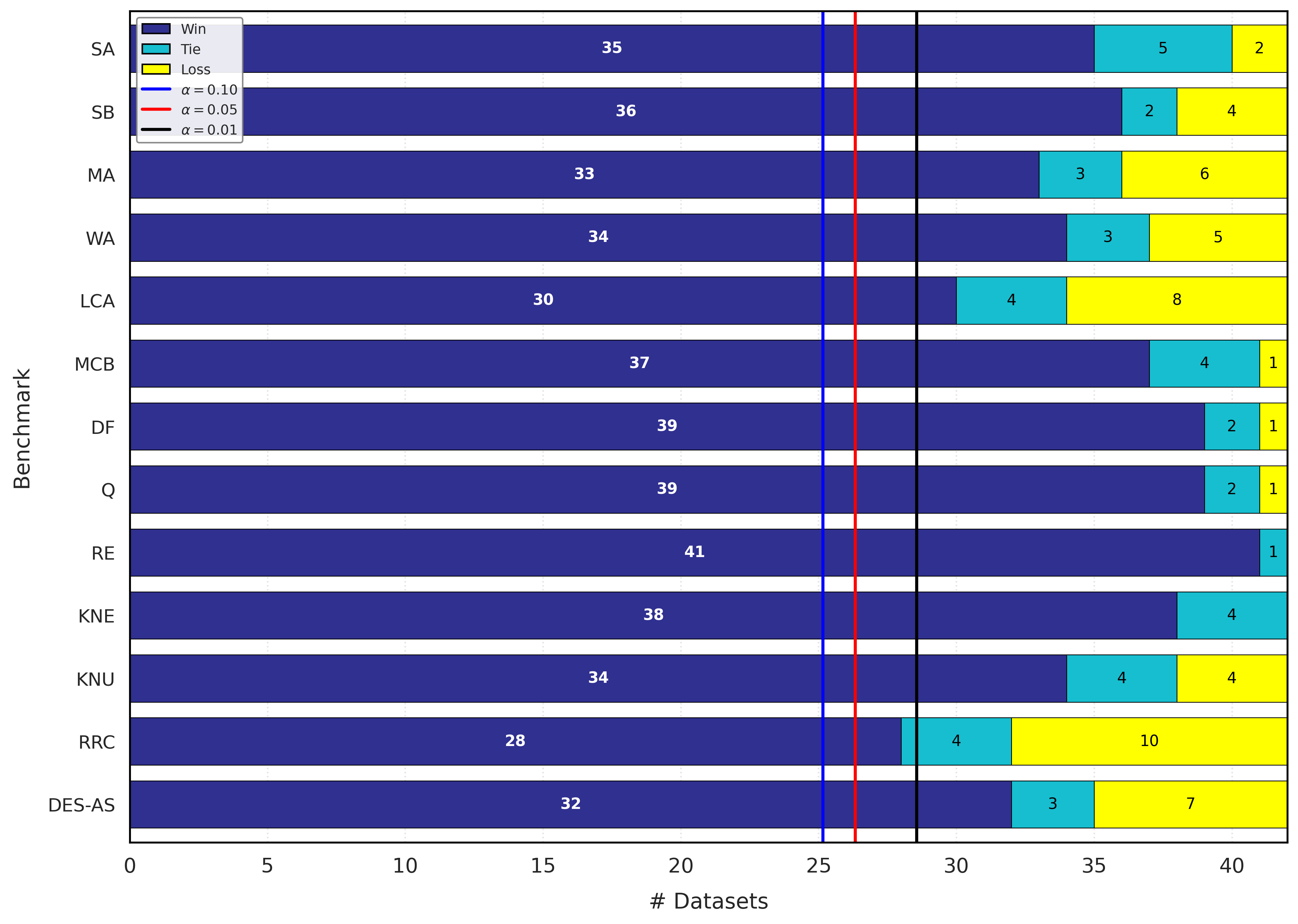}
    \caption{Win-tie-loss distribution of the average classification accuracy of the proposed BPE vs. thirteen baseline models over 42 datasets in \textbf{heterogeneous} ensemble scenario. Each full line illustrates the critical value $n_c = \{24.05, 25.20, 27.36\}$ considering confidence level of $\alpha = \{0.10, 0.05, 0.01\}$.}
    \label{fig:wtl_hetero}
\end{figure}

\subsection{Homogeneous ensemble}
In the homogeneous setting, the pool consists of $M=40$ decision trees trained with identical hyperparameters, where diversity is introduced solely by Bagging (each tree is fitted on a different bootstrap resample of the training set).

In this Bagging-only diversity setting, BPE remains highly competitive and achieves the best overall performance. As shown in Table \ref{tab:homogeneous_oof_acc}, BPE achieves the highest average accuracy of 84.06\%, slightly outperforming the strongest static baselines (Simple Average and Weighted Average, both 84.00\%). The Friedman test results in Fig. \ref{fig:homogeneous_friedman} provide a consistent conclusion: BPE achieves the best average rank of 2.881, followed by Weighted Average with an average rank of 3.452.

The statistical analysis presented in Table \ref{tab:homogeneous_oof_wilcoxon} reveals a pattern distinct from the heterogeneous case. BPE is statistically significantly better than the best single classifier ($p=0.0000$), Median Average ($p=0.0000$), Weighted Average ($p=0.0246$), and all evaluated DES/DCS baselines (all $p < 0.05$). However, the null hypothesis is not rejected when comparing BPE with Simple Average ($p=0.2735$). This result is expected in homogeneous pools: Bagging-style averaging is already a particularly strong and stable fusion rule for decision trees, so additional instance-wise weighting may yield only marginal gains on average while still maintaining strong competitiveness.

\begin{table}[htbp]
\centering
\caption{Average classification accuracy obtained by BPE and baseline models over 42 datasets in \textbf{homogeneous (OOF)} ensemble case. The best results are highlighted in \textbf{bold-face}.}
\label{tab:homogeneous_oof_acc}
\setlength{\tabcolsep}{2pt}
\scriptsize
\resizebox{\textwidth}{!}{
\begin{tabular}{l cccc cc ccccccc c}
\toprule
Dataset & \multicolumn{4}{c}{Static} & \multicolumn{2}{c}{DCS} & \multicolumn{7}{c}{DES} & \multicolumn{1}{c}{BPE} \\
\cmidrule(lr){2-5} \cmidrule(lr){6-7} \cmidrule(lr){8-14} \cmidrule(lr){15-15}
 & SB & SA & MA & WA & LCA & MCB & DES-AS & DF & Q & RE & KNE & KNU & RRC & Entropy \\
\midrule
D1 & 63.41 & \textbf{71.05} & 70.92 & 70.62 & 70.39 & 70.20 & 68.82 & 67.87 & 69.15 & 68.36 & 69.90 & 69.90 & 70.23 & 70.85 \\
D2 & 78.28 & 80.88 & 80.19 & \textbf{80.93} & 80.56 & 80.28 & 80.19 & 79.35 & 79.26 & 79.91 & 80.37 & 80.51 & 80.79 & 80.88 \\
D3 & 77.95 & 81.98 & 81.51 & 82.09 & 82.11 & 81.77 & 81.82 & 80.65 & 80.77 & 80.02 & 81.30 & 82.00 & 82.13 & \textbf{82.31} \\
D4 & 82.27 & \textbf{85.79} & 85.54 & 85.77 & 85.72 & 85.27 & 85.58 & 84.39 & 84.67 & 84.35 & 84.66 & 85.60 & 85.70 & 85.70 \\
D5 & 96.16 & 97.06 & 97.01 & \textbf{97.07} & \textbf{97.07} & 97.06 & 97.03 & 96.89 & 96.93 & 96.84 & 96.91 & 97.05 & \textbf{97.07} & 97.04 \\
D6 & 85.79 & 89.94 & 89.71 & \textbf{90.01} & 89.85 & 89.84 & 89.80 & 89.20 & 89.39 & 88.94 & 89.10 & 89.68 & 89.83 & 89.96 \\
D7 & 77.49 & 79.66 & 79.57 & 80.51 & 79.96 & 80.26 & 79.83 & 79.36 & 79.45 & 78.55 & 79.40 & 80.17 & 79.91 & \textbf{80.98} \\
D8 & 86.87 & 89.45 & 89.19 & \textbf{89.55} & 89.51 & 89.51 & \textbf{89.55} & 88.96 & 89.08 & 88.36 & 88.83 & 89.53 & 89.47 & \textbf{89.55} \\
D9 & 84.10 & \textbf{89.64} & 89.57 & \textbf{89.64} & 89.58 & 89.55 & 89.23 & 88.97 & 89.14 & 89.06 & 88.37 & 89.45 & 89.58 & 89.45 \\
D10 & 80.36 & 86.38 & 83.72 & 86.86 & 86.44 & 86.65 & 85.65 & 84.72 & 84.78 & 84.12 & 85.47 & 85.81 & 86.54 & \textbf{87.74} \\
D11 & 77.54 & 81.96 & 81.98 & 82.42 & \textbf{82.69} & 81.79 & 82.08 & 80.23 & 80.88 & 80.57 & 81.28 & 82.21 & \textbf{82.69} & 82.44 \\
D12 & 70.66 & 76.22 & 75.97 & 76.26 & 75.81 & 74.97 & 74.91 & 72.72 & 73.26 & 72.88 & 73.97 & 75.54 & 75.87 & \textbf{76.43} \\
D13 & 98.64 & 99.04 & 99.03 & 99.05 & \textbf{99.06} & \textbf{99.06} & 99.03 & 98.94 & 98.95 & 98.87 & 98.97 & \textbf{99.06} & \textbf{99.06} & \textbf{99.06} \\
D14 & 98.73 & 99.12 & 99.12 & 99.13 & 99.13 & \textbf{99.14} & 99.12 & 99.03 & 99.04 & 98.99 & 99.02 & 99.13 & \textbf{99.14} & 99.13 \\
D15 & 87.72 & \textbf{89.42} & 89.37 & 89.26 & 89.24 & 89.19 & 89.07 & 89.01 & 89.00 & 88.95 & 88.95 & 89.24 & 89.23 & 89.41 \\
D16 & 85.44 & 88.57 & 88.49 & 88.46 & 88.39 & 88.29 & 88.20 & 88.20 & 88.24 & 87.99 & 88.05 & 88.45 & 88.39 & \textbf{88.60} \\
D17 & 74.88 & 81.31 & 81.18 & 81.31 & 81.31 & 81.19 & 81.13 & 81.09 & 81.17 & 81.12 & 79.72 & 81.28 & 81.31 & \textbf{81.41} \\
D18 & 79.38 & 82.60 & 82.37 & 82.45 & 82.17 & 82.20 & 81.69 & 81.40 & 81.57 & 81.13 & 82.05 & 82.09 & 82.14 & \textbf{82.68} \\
D19 & 81.65 & 85.14 & 84.99 & 85.07 & 84.89 & 84.68 & 84.87 & 84.68 & 84.69 & 84.59 & 84.48 & 84.82 & 84.89 & \textbf{85.24} \\
D20 & 84.34 & \textbf{88.21} & 88.13 & 88.08 & 88.11 & 87.79 & 87.50 & 87.58 & 87.55 & 87.50 & 87.76 & 87.95 & 87.95 & 88.18 \\
D21 & 87.53 & 89.58 & 89.40 & \textbf{89.59} & 89.55 & 89.43 & 89.39 & 89.21 & 89.26 & 89.21 & 89.38 & 89.47 & 89.55 & 89.58 \\
D22 & 87.34 & 89.79 & 89.72 & 89.78 & 89.81 & 89.76 & 89.76 & 89.65 & 89.72 & 89.69 & 89.32 & 89.79 & 89.81 & \textbf{89.84} \\
D23 & 73.40 & 76.33 & 75.93 & 76.27 & 75.93 & \textbf{76.67} & 75.47 & 75.33 & 76.00 & 75.47 & 74.73 & 75.73 & 76.13 & 76.47 \\
D24 & 29.94 & \textbf{37.44} & 37.31 & 36.38 & 36.19 & 36.25 & 33.87 & 34.50 & 35.19 & 35.06 & 37.38 & 34.75 & 35.81 & 37.00 \\
D25 & 86.96 & 89.96 & 89.77 & 90.14 & 90.25 & 90.27 & 90.04 & 88.99 & 89.00 & 88.53 & 89.16 & 90.28 & 90.27 & \textbf{90.33} \\
D26 & 65.92 & 68.81 & \textbf{68.95} & 68.49 & 68.57 & 68.34 & 68.56 & 68.26 & 68.15 & 67.98 & 68.35 & 68.65 & 68.56 & 68.65 \\
D27 & 91.63 & \textbf{93.76} & 93.68 & 93.68 & 93.68 & 93.63 & 93.53 & 93.46 & 93.45 & 93.29 & 93.39 & 93.62 & 93.66 & 93.69 \\
D28 & 83.47 & 88.47 & 88.38 & 88.47 & 88.53 & 88.41 & 88.21 & 88.04 & 88.14 & 88.11 & 87.05 & 88.46 & \textbf{88.56} & 88.45 \\
D29 & 80.33 & \textbf{84.53} & 84.37 & 84.38 & 84.40 & 84.33 & 83.98 & 83.32 & 83.29 & 83.16 & 83.33 & 84.25 & 84.42 & 84.34 \\
D30 & 87.31 & 93.55 & 93.55 & 93.52 & 93.54 & 93.38 & 93.29 & 92.83 & 92.96 & 92.93 & 91.96 & 93.47 & 93.53 & \textbf{93.57} \\
D31 & 87.53 & 89.68 & 89.55 & 89.74 & 90.00 & 89.70 & 89.77 & 89.33 & 89.29 & 88.42 & 89.26 & 89.94 & 89.98 & \textbf{90.24} \\
D32 & 47.26 & 63.76 & 63.41 & \textbf{63.77} & 63.57 & 63.57 & 62.31 & 62.66 & 63.30 & 62.86 & 61.16 & 63.02 & 63.59 & 63.04 \\
D33 & 67.04 & 72.94 & 72.46 & 72.99 & \textbf{73.05} & 72.51 & 71.60 & 70.20 & 70.81 & 70.90 & 71.88 & 72.47 & 73.04 & 72.91 \\
D34 & 89.30 & 93.50 & 93.32 & 93.49 & 93.53 & 93.47 & 93.43 & 92.73 & 92.82 & 92.56 & 92.55 & 93.53 & \textbf{93.54} & 93.41 \\
D35 & 99.71 & 99.70 & 99.69 & 99.71 & 99.71 & 99.71 & 99.71 & 99.69 & 99.69 & \textbf{99.72} & 99.70 & 99.71 & 99.71 & 99.71 \\
D36 & 94.57 & 98.57 & 98.11 & 98.69 & 98.71 & 98.47 & 98.62 & 97.47 & 97.51 & 96.75 & 97.44 & 98.68 & 98.68 & \textbf{98.75} \\
D37 & 89.91 & 92.01 & 91.85 & 92.05 & 91.97 & 91.76 & 91.96 & 91.46 & 91.41 & 91.26 & 91.76 & 91.94 & 91.97 & \textbf{92.22} \\
D38 & 75.83 & \textbf{79.34} & 78.97 & 79.31 & 79.24 & 79.17 & 79.13 & 78.79 & 78.87 & 78.47 & 78.83 & 79.26 & 79.24 & 79.33 \\
D39 & 76.17 & 87.39 & 87.00 & 87.21 & 87.05 & \textbf{87.43} & 85.95 & 84.19 & 84.88 & 84.24 & 84.61 & 86.34 & 86.89 & 86.80 \\
D40 & 82.05 & 85.16 & 84.97 & 85.23 & 85.11 & 85.00 & 84.99 & 84.37 & 84.54 & 84.56 & 84.47 & 85.09 & 85.15 & \textbf{85.25} \\
D41 & 69.28 & 76.23 & 75.99 & 76.31 & 76.11 & 75.25 & 74.72 & 73.14 & 73.46 & 73.37 & 73.97 & 75.44 & 76.09 & \textbf{76.45} \\
D42 & 44.82 & \textbf{54.14} & 52.52 & 54.07 & 53.07 & 53.02 & 51.32 & 51.09 & 51.86 & 51.16 & 53.61 & 52.25 & 52.86 & 53.30 \\
\midrule
Average & 79.74 & 84.00 & 83.73 & 84.00 & 83.89 & 83.77 & 83.45 & 82.90 & 83.11 & 82.83 & 83.14 & 83.70 & 83.88 & \textbf{84.06} \\
\bottomrule
\end{tabular}
}
\end{table}

\begin{table}[htbp]
\centering
\caption{Wilcoxon signed-rank test results for comparing BPE with thirteen baseline models in \textbf{homogeneous (OOF)} ensemble case.}
\label{tab:homogeneous_oof_wilcoxon}
\begin{tabular*}{\textwidth}{@{\extracolsep{\fill}} ll cccc }
\toprule
Ensemble type & Comparison & $R^+$ & $R^-$ & Hypothesis & $p$-value \\
\midrule
\multirow{4}{*}{Static} 
 & BPE vs. SB     & 902.0 & 1.0   & Rejected for BPE at 5\% & 0.0000** \\
 & BPE vs. SA   & 515.0 & 346.0 & Not Rejected            & 0.2735 \\
 & BPE vs. MA   & 780.0 & 123.0 & Rejected for BPE at 5\% & 0.0000** \\
 & BPE vs. WA & 604.0 & 257.0 & Rejected for BPE at 5\% & 0.0246** \\
\midrule
\multirow{2}{*}{DCS}    
 & BPE vs. LCA      & 703.0 & 200.0 & Rejected for BPE at 5\% & 0.0013** \\
 & BPE vs. MCB      & 790.0 & 113.0 & Rejected for BPE at 5\% & 0.0000** \\
\midrule
\multirow{7}{*}{DES}    
 & BPE vs. DES-AS   & 857.0 & 4.0   & Rejected for BPE at 5\% & 0.0000** \\
 & BPE vs. DF       & 903.0 & 0.0   & Rejected for BPE at 5\% & 0.0000** \\
 & BPE vs. Q        & 895.0 & 8.0   & Rejected for BPE at 5\% & 0.0000** \\
 & BPE vs. RE       & 902.0 & 1.0   & Rejected for BPE at 5\% & 0.0000** \\
 & BPE vs. KNE      & 886.0 & 17.0  & Rejected for BPE at 5\% & 0.0000** \\
 & BPE vs. KNU      & 798.0 & 22.0  & Rejected for BPE at 5\% & 0.0000** \\
 & BPE vs. RRC      & 650.0 & 170.0 & Rejected for BPE at 5\% & 0.0013** \\
\bottomrule
\addlinespace
\multicolumn{6}{l}{** Close to the $p$-value means that statistical differences are found with $\alpha = 0.05$.} \\
\end{tabular*}
\end{table}

\begin{figure}[t]
    \centering
    \includegraphics[width=0.9\linewidth]{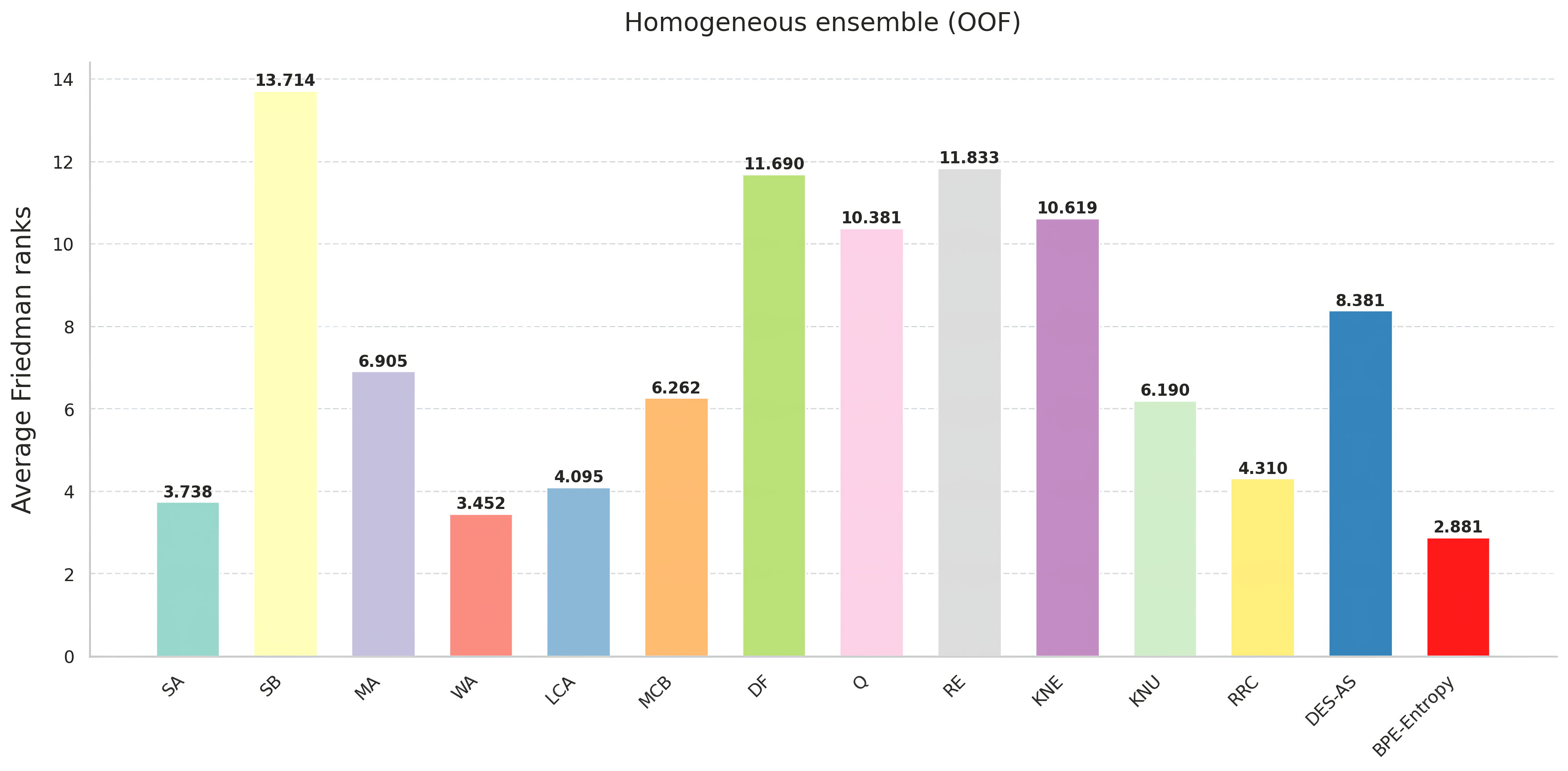}
    \caption{Average Friedman ranks of BPE and thirteen baseline models in \textbf{homogeneous} ensemble case.}
    \label{fig:homogeneous_friedman}
\end{figure}

\begin{figure}[t]
    \centering
    \includegraphics[width=0.9\linewidth]{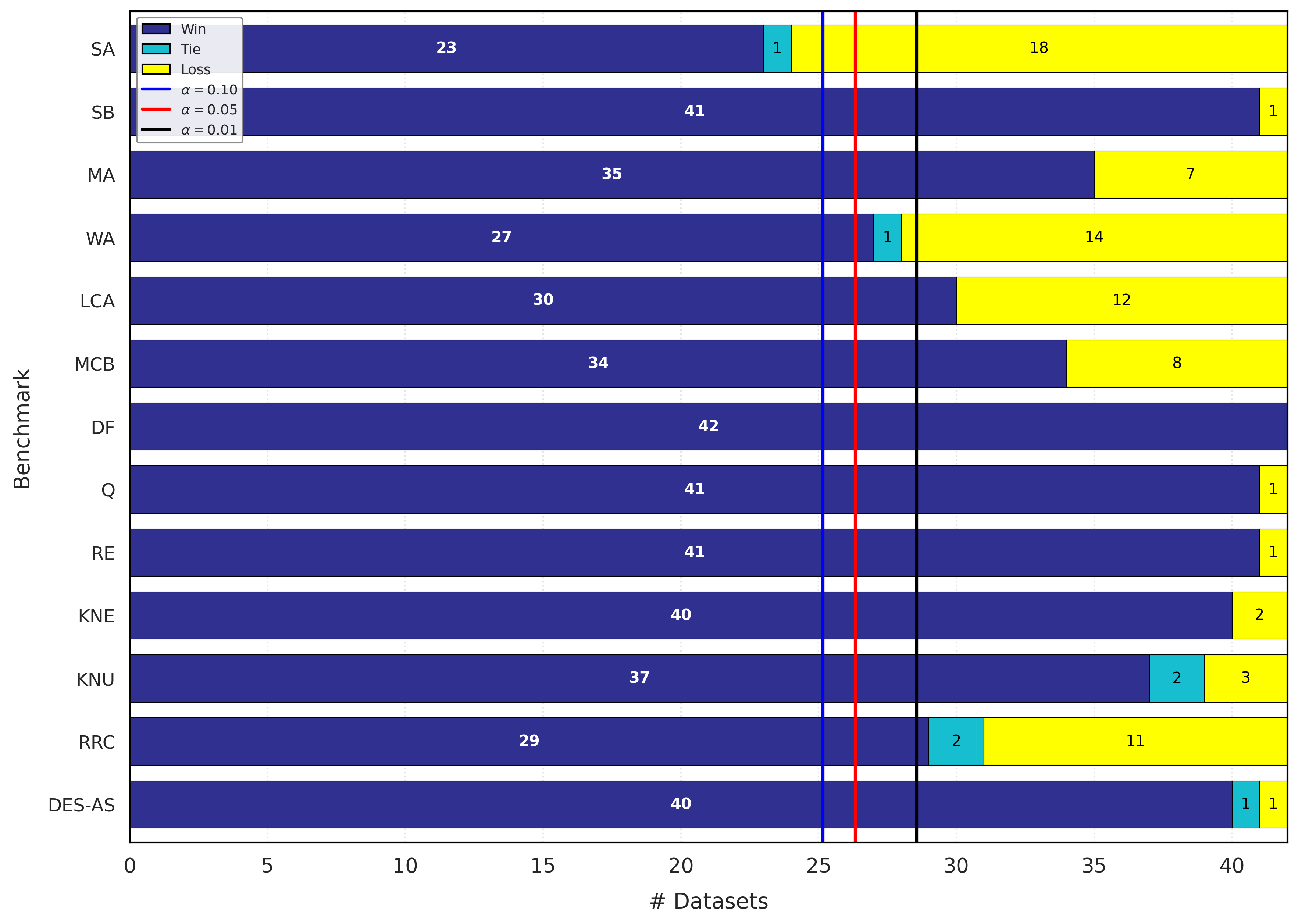}
    \caption{Win-tie-loss distribution of the average classification accuracy of the proposed BPE vs. thirteen baseline models over 42 datasets in \textbf{homogeneous} ensemble scenario. Each full line illustrates the critical value $n_c = \{24.05, 25.20, 27.36\}$ considering confidence level of $\alpha = \{0.10, 0.05, 0.01\}$.}
    \label{fig:wtl_homo}
\end{figure}

\subsection{Hyperparameter Sensitivity Analysis}
We conduct a sensitivity analysis for BPE-Entropy, and additionally study the tolerance parameter $\alpha$ used in our screening protocol. In practical applications, strong performance is often obtained by ensembling multiple high-capacity learners; empirically, our heterogeneous ensemble setting achieves higher accuracy than the homogeneous weak-learner setting. Therefore, all sensitivity results in this subsection are reported under the heterogeneous ensemble protocol.

We focus on two key hyperparameters of BPE-Entropy: the perturbation scale $\delta$ used in offline profiling and the sensitivity factor $\lambda$ used in the exponential mapping from normalized scores to aggregation weights. We also vary the screening tolerance $\alpha$ in the experimental protocol, which controls how aggressively clearly underperforming base learners are pruned when forming the candidate pool; thus, each $\alpha$ induces a different base-model pool. Unless otherwise specified, we vary one hyperparameter at a time, keep the remaining settings identical to the default protocol, and report the average classification accuracy over the 42 datasets.

\begin{table}[htbp]
\centering
\caption{Detailed accuracy analysis with different values of Sensitivity Factor $\lambda$ and Perturbation Scale $\delta$. The default settings are $\lambda=1$ and $\delta=0.5$. The best results are highlighted in \textbf{bold-face}.}
\label{tab:hyperparam_analysis}
\small
\begin{tabular*}{\textwidth}{@{\extracolsep{\fill}} l cccc c cccc }
\toprule
\multirow{2}{*}{Dataset} & \multicolumn{4}{c}{Sensitivity Factor $\lambda$} & \multirow{2}{*}{Default} & \multicolumn{4}{c}{Perturbation Scale $\delta$} \\
\cmidrule(lr){2-5} \cmidrule(lr){7-10}
 & 0.5 & 0.7 & 1.2 & 1.5 & ($\lambda=1, \delta=0.5$) & 0.1 & 0.3 & 0.7 & 1.0 \\
\midrule
D1 & 0.7925 & 0.7931 & 0.7934 & 0.7928 & 0.7928 & 0.7892 & 0.7921 & 0.7931 & \textbf{0.7941} \\
D2 & 0.8600 & 0.8595 & 0.8595 & 0.8591 & 0.8586 & 0.8567 & 0.8586 & 0.8586 & \textbf{0.8605} \\
D3 & 0.8413 & 0.8409 & 0.8398 & 0.8395 & 0.8407 & 0.8373 & 0.8393 & 0.8411 & \textbf{0.8418} \\
D4 & 0.8733 & 0.8730 & 0.8733 & 0.8735 & 0.8732 & 0.8729 & \textbf{0.8736} & 0.8730 & 0.8726 \\
D5 & 0.9729 & 0.9729 & 0.9730 & \textbf{0.9731} & 0.9730 & 0.9730 & 0.9729 & 0.9730 & \textbf{0.9731} \\
D6 & 0.9339 & 0.9339 & 0.9342 & 0.9341 & 0.9340 & 0.9315 & 0.9334 & \textbf{0.9346} & 0.9335 \\
D7 & 0.8357 & 0.8370 & 0.8374 & 0.8374 & 0.8374 & \textbf{0.8387} & 0.8366 & 0.8366 & 0.8357 \\
D8 & 0.9308 & 0.9313 & 0.9313 & 0.9309 & 0.9313 & 0.9283 & 0.9309 & 0.9323 & \textbf{0.9332} \\
D9 & \textbf{0.9209} & \textbf{0.9209} & 0.9204 & 0.9204 & 0.9207 & 0.9198 & 0.9205 & 0.9204 & 0.9198 \\
D10 & \textbf{0.9899} & \textbf{0.9899} & \textbf{0.9899} & \textbf{0.9899} & \textbf{0.9899} & 0.9895 & \textbf{0.9899} & \textbf{0.9899} & \textbf{0.9899} \\
D11 & 0.8806 & \textbf{0.8813} & 0.8808 & 0.8800 & \textbf{0.8813} & 0.8792 & 0.8808 & \textbf{0.8813} & 0.8802 \\
D12 & 0.8060 & 0.8063 & 0.8056 & 0.8059 & 0.8065 & 0.8051 & 0.8063 & \textbf{0.8066} & 0.8062 \\
D13 & \textbf{0.9944} & \textbf{0.9944} & 0.9943 & 0.9943 & 0.9943 & 0.9942 & \textbf{0.9944} & 0.9943 & 0.9943 \\
D14 & \textbf{0.9947} & \textbf{0.9947} & 0.9946 & 0.9946 & 0.9946 & \textbf{0.9947} & \textbf{0.9947} & 0.9946 & 0.9946 \\
D15 & 0.9121 & \textbf{0.9124} & 0.9117 & 0.9114 & 0.9119 & 0.9110 & 0.9113 & 0.9119 & 0.9117 \\
D16 & 0.9006 & 0.9004 & \textbf{0.9007} & \textbf{0.9007} & 0.9006 & 0.9002 & 0.9004 & 0.9005 & 0.9004 \\
D17 & 0.8189 & 0.8188 & 0.8189 & \textbf{0.8190} & 0.8189 & 0.8184 & 0.8188 & 0.8187 & 0.8185 \\
D18 & 0.8364 & 0.8359 & 0.8363 & 0.8362 & 0.8364 & \textbf{0.8375} & 0.8367 & 0.8363 & 0.8364 \\
D19 & 0.8655 & 0.8654 & \textbf{0.8656} & 0.8654 & 0.8653 & 0.8652 & 0.8655 & 0.8654 & 0.8645 \\
D20 & \textbf{0.8971} & 0.8968 & 0.8942 & 0.8937 & 0.8950 & 0.8947 & 0.8950 & 0.8955 & 0.8953 \\
D21 & 0.9008 & 0.9009 & \textbf{0.9013} & 0.9012 & \textbf{0.9013} & 0.9010 & 0.9012 & 0.9011 & 0.9011 \\
D22 & \textbf{0.9039} & 0.9037 & 0.9036 & 0.9033 & 0.9037 & 0.9035 & 0.9036 & 0.9035 & 0.9034 \\
D23 & 0.8200 & 0.8200 & 0.8213 & 0.8213 & 0.8200 & 0.8147 & 0.8187 & 0.8213 & \textbf{0.8233} \\
D24 & 0.4219 & 0.4244 & 0.4206 & 0.4213 & 0.4231 & \textbf{0.4313} & 0.4250 & 0.4219 & 0.4213 \\
D25 & 0.9109 & 0.9109 & 0.9109 & 0.9106 & 0.9110 & 0.9111 & 0.9106 & 0.9110 & \textbf{0.9112} \\
D26 & 0.7138 & 0.7144 & 0.7143 & 0.7134 & 0.7139 & 0.7150 & \textbf{0.7153} & 0.7134 & 0.7131 \\
D27 & 0.9452 & 0.9452 & 0.9451 & 0.9451 & 0.9452 & 0.9450 & 0.9451 & \textbf{0.9453} & \textbf{0.9453} \\
D28 & 0.9053 & 0.9055 & 0.9060 & 0.9060 & 0.9058 & 0.9051 & 0.9057 & 0.9059 & \textbf{0.9061} \\
D29 & 0.8766 & 0.8765 & 0.8760 & 0.8762 & 0.8762 & \textbf{0.8770} & 0.8762 & 0.8760 & 0.8756 \\
D30 & \textbf{0.9783} & 0.9782 & 0.9782 & \textbf{0.9783} & 0.9782 & \textbf{0.9783} & \textbf{0.9783} & 0.9781 & 0.9782 \\
D31 & 0.9165 & 0.9158 & 0.9158 & 0.9158 & 0.9159 & 0.9164 & \textbf{0.9168} & 0.9159 & 0.9155 \\
D32 & 0.6826 & 0.6834 & \textbf{0.6842} & \textbf{0.6842} & 0.6840 & 0.6815 & 0.6833 & \textbf{0.6842} & 0.6835 \\
D33 & 0.8042 & 0.8049 & 0.8075 & 0.8081 & 0.8070 & 0.7980 & 0.8036 & 0.8073 & \textbf{0.8106} \\
D34 & \textbf{0.9643} & \textbf{0.9643} & 0.9641 & 0.9638 & 0.9642 & 0.9638 & 0.9641 & 0.9642 & 0.9642 \\
D35 & \textbf{0.9986} & \textbf{0.9986} & 0.9985 & 0.9985 & 0.9985 & 0.9984 & 0.9985 & \textbf{0.9986} & \textbf{0.9986} \\
D36 & \textbf{1.0000} & \textbf{1.0000} & \textbf{1.0000} & \textbf{1.0000} & \textbf{1.0000} & \textbf{1.0000} & \textbf{1.0000} & \textbf{1.0000} & \textbf{1.0000} \\
D37 & 0.9286 & 0.9284 & 0.9293 & 0.9290 & 0.9292 & 0.9290 & 0.9290 & 0.9290 & \textbf{0.9295} \\
D38 & 0.8126 & \textbf{0.8129} & \textbf{0.8129} & 0.8128 & \textbf{0.8129} & 0.8099 & 0.8113 & 0.8128 & 0.8121 \\
D39 & 0.9962 & 0.9963 & 0.9962 & 0.9959 & 0.9962 & 0.9948 & 0.9946 & \textbf{0.9969} & \textbf{0.9969} \\
D40 & 0.8667 & \textbf{0.8677} & 0.8676 & 0.8673 & 0.8676 & 0.8660 & 0.8671 & 0.8674 & 0.8669 \\
D41 & 0.8226 & 0.8226 & 0.8227 & 0.8231 & 0.8226 & 0.8224 & \textbf{0.8235} & 0.8225 & 0.8227 \\
D42 & 0.5795 & 0.5791 & 0.5807 & 0.5807 & 0.5789 & 0.5725 & 0.5777 & 0.5805 & \textbf{0.5814} \\
\midrule
\textbf{Average} & 0.8716 & 0.8717 & 0.8717 & 0.8716 & 0.8717 & 0.8708 & 0.8715 & \textbf{0.8718} & \textbf{0.8718} \\
\bottomrule
\end{tabular*}
\end{table}

Table \ref{tab:hyperparam_analysis} shows that BPE-Entropy is broadly insensitive to moderate changes in $\lambda$ and $\delta$. On average, varying $\lambda$ within $[0.5,1.5]$ leads to only negligible differences (all averages are essentially identical, around $0.8716$--$0.8717$), indicating that the exponential mapping is not overly sensitive to its temperature-like scaling.

In contrast, $\delta$ has a slightly more noticeable (but still mild) influence: the average accuracy increases from $0.8708$ at $\delta=0.1$ to $0.8718$ at $\delta\in\{0.7,1.0\}$. This suggests that using moderate perturbations improves the fidelity of behavioral profiling, whereas overly small perturbations may not sufficiently reveal model response variation. Importantly, the default setting ($\lambda=1,\delta=0.5$) remains close to the best-performing configuration, supporting the robustness of our default hyperparameters.

\begin{table}[htbp]
\centering
\caption{Average classification accuracy obtained by Top-3 algorithms under different tolerance levels ($\alpha$) in \textbf{heterogeneous (OOF)} ensemble case. The best results in each group are highlighted in \textbf{boldface}.}
\label{tab:tolerance_comparison}
\setlength{\tabcolsep}{2pt}
\resizebox{\textwidth}{!}{
\begin{tabular}{l ccc ccc ccc ccc ccc}
\toprule
Dataset & \multicolumn{3}{c}{$\alpha=0.05$} & \multicolumn{3}{c}{$\alpha=0.1$} & \multicolumn{3}{c}{$\alpha=0.15$} & \multicolumn{3}{c}{$\alpha=0.2$} & \multicolumn{3}{c}{$\alpha=0.25$} \\
\cmidrule(lr){2-4} \cmidrule(lr){5-7} \cmidrule(lr){8-10} \cmidrule(lr){11-13} \cmidrule(lr){14-16}
& BPE & RRC & LCA & BPE & RRC & LCA & BPE & RRC & LCA & BPE & Weighte & RRC & BPE & Weighte & RRC \\
\midrule
D1 & \textbf{78.89} & \textbf{78.89} & \textbf{78.89} & 78.92 & \textbf{78.95} & 78.89 & 79.28 & \textbf{79.38} & 79.34 & \textbf{79.67} & 79.05 & 79.11 & \textbf{79.87} & 79.08 & 79.18 \\
D2 & 85.86 & 86.09 & \textbf{86.19} & 85.86 & 86.09 & \textbf{86.19} & 85.86 & 86.09 & \textbf{86.19} & 85.86 & 86.05 & \textbf{86.14} & 85.95 & \textbf{86.19} & 86.14 \\
D3 & 83.84 & 83.77 & 83.80 & 83.73 & 83.66 & 83.71 & \textbf{84.07} & 83.93 & 83.95 & \textbf{84.23} & 84.05 & 84.02 & \textbf{84.13} & 83.95 & 83.91 \\
D4 & 87.32 & 87.33 & \textbf{87.37} & 87.20 & 87.28 & \textbf{87.28} & 87.32 & \textbf{87.37} & 87.34 & 87.42 & \textbf{87.48} & 87.37 & 87.46 & \textbf{87.52} & 87.49 \\
D5 & 97.29 & \textbf{97.31} & \textbf{97.31} & 97.29 & \textbf{97.30} & \textbf{97.30} & \textbf{97.30} & \textbf{97.30} & 97.29 & 97.32 & \textbf{97.33} & 97.32 & 97.32 & \textbf{97.33} & \textbf{97.33} \\
D6 & \textbf{93.45} & 93.29 & 93.30 & \textbf{93.38} & 93.26 & 93.27 & \textbf{93.40} & 93.22 & 93.23 & \textbf{93.38} & 93.34 & 93.34 & \textbf{93.40} & 93.36 & 93.36 \\
D7 & \textbf{83.66} & 83.57 & 83.53 & \textbf{83.66} & 83.57 & 83.53 & \textbf{83.74} & 83.49 & 83.45 & \textbf{83.74} & 83.57 & 83.62 & 83.32 & 83.32 & \textbf{83.36} \\
D8 & \textbf{93.13} & 92.92 & 92.91 & \textbf{93.13} & 92.92 & 92.91 & \textbf{93.13} & 92.91 & 92.89 & \textbf{93.13} & 92.83 & 92.89 & \textbf{93.08} & 92.81 & 92.85 \\
D9 & 91.90 & \textbf{91.95} & \textbf{91.95} & 91.98 & \textbf{92.06} & 92.05 & 92.07 & 92.08 & 92.08 & \textbf{92.13} & 92.09 & 92.12 & \textbf{92.15} & 92.09 & 92.14 \\
D10 & \textbf{98.99} & \textbf{98.99} & \textbf{98.99} & \textbf{98.99} & \textbf{98.99} & \textbf{98.99} & \textbf{98.99} & \textbf{98.99} & \textbf{98.99} & 99.90 & 99.90 & \textbf{99.92} & 99.90 & 99.93 & \textbf{99.94} \\
D11 & \textbf{87.43} & 87.41 & 87.41 & \textbf{87.66} & 87.56 & 87.56 & \textbf{88.13} & 88.02 & 88.02 & 88.27 & \textbf{88.44} & 88.27 & 88.34 & \textbf{88.40} & 88.25 \\
D12 & \textbf{79.29} & 78.78 & 78.79 & 80.00 & \textbf{80.06} & \textbf{80.06} & \textbf{80.65} & 80.51 & 80.49 & 80.82 & \textbf{80.88} & 80.68 & \textbf{81.21} & 81.10 & 81.07 \\
D13 & \textbf{99.42} & \textbf{99.42} & \textbf{99.42} & \textbf{99.43} & 99.42 & 99.42 & \textbf{99.43} & 99.42 & 99.42 & \textbf{99.44} & 99.43 & 99.43 & \textbf{99.45} & 99.44 & 99.44 \\
D14 & \textbf{99.45} & 99.44 & 99.44 & \textbf{99.47} & 99.46 & 99.46 & \textbf{99.46} & \textbf{99.46} & \textbf{99.46} & \textbf{99.46} & 99.45 & \textbf{99.46} & \textbf{99.47} & 99.46 & 99.46 \\
D15 & 91.04 & 91.10 & \textbf{91.10} & 91.08 & 91.13 & \textbf{91.14} & 91.19 & \textbf{91.23} & \textbf{91.23} & \textbf{91.20} & 91.18 & \textbf{91.20} & 91.22 & \textbf{91.25} & 91.24 \\
D16 & \textbf{89.79} & 89.59 & 89.59 & \textbf{89.89} & 89.69 & 89.70 & \textbf{90.06} & 89.98 & 89.99 & 90.09 & \textbf{90.10} & 90.02 & \textbf{90.05} & \textbf{90.05} & 90.00 \\
D17 & \textbf{81.75} & 81.66 & 81.63 & \textbf{81.85} & 81.71 & 81.64 & \textbf{81.89} & 81.73 & 81.66 & \textbf{81.89} & 81.49 & 81.74 & \textbf{81.90} & 81.50 & 81.75 \\
D18 & 83.59 & \textbf{83.62} & 83.55 & \textbf{83.71} & 83.55 & 83.52 & 83.64 & \textbf{83.68} & 83.62 & 83.73 & 83.89 & \textbf{83.91} & 83.82 & 83.75 & \textbf{83.81} \\
D19 & \textbf{86.27} & 86.26 & 86.24 & \textbf{86.53} & 86.50 & 86.46 & 86.53 & \textbf{86.54} & 86.45 & 86.62 & 86.49 & \textbf{86.62} & \textbf{86.63} & 86.52 & \textbf{86.64} \\
D20 & \textbf{89.47} & 89.42 & 89.50 & \textbf{89.47} & 89.42 & 89.50 & 89.50 & 89.45 & \textbf{89.53} & \textbf{89.61} & 89.47 & 89.45 & \textbf{89.63} & 89.39 & 89.53 \\
D21 & \textbf{90.08} & 89.98 & 89.98 & \textbf{90.11} & 90.04 & 90.05 & \textbf{90.13} & 90.01 & 90.01 & \textbf{90.10} & 90.05 & 90.01 & \textbf{90.10} & 90.03 & 89.99 \\
D22 & 90.33 & \textbf{90.36} & \textbf{90.36} & 90.38 & \textbf{90.41} & 90.40 & \textbf{90.37} & 90.34 & 90.34 & 90.29 & \textbf{90.34} & 90.31 & \textbf{90.31} & \textbf{90.31} & 90.29 \\
D23 & \textbf{82.07} & 81.53 & 81.33 & \textbf{82.07} & 81.53 & 81.33 & \textbf{82.00} & 81.53 & 81.33 & \textbf{82.07} & 81.47 & 81.47 & \textbf{82.00} & 81.40 & 81.33 \\
D24 & \textbf{41.81} & 41.75 & 41.88 & \textbf{42.00} & 41.94 & \textbf{42.00} & 42.31 & \textbf{42.63} & 42.50 & 42.56 & 43.13 & \textbf{43.06} & 43.63 & \textbf{43.63} & 43.31 \\
D25 & \textbf{91.22} & 91.09 & 91.09 & \textbf{91.22} & 91.09 & 91.09 & \textbf{91.10} & 91.00 & 91.00 & \textbf{91.23} & 91.05 & 91.04 & \textbf{91.25} & 91.07 & 91.05 \\
D26 & 70.81 & \textbf{71.11} & 70.97 & \textbf{71.31} & 71.25 & 71.26 & \textbf{71.39} & 71.29 & 71.31 & \textbf{71.33} & 71.07 & 71.17 & 70.81 & 70.86 & \textbf{71.13} \\
D27 & 94.49 & \textbf{94.52} & \textbf{94.52} & 94.41 & 94.46 & \textbf{94.47} & \textbf{94.52} & \textbf{94.52} & \textbf{94.52} & \textbf{94.61} & 94.56 & 94.54 & \textbf{94.63} & 94.58 & 94.58 \\
D28 & \textbf{90.07} & 89.98 & 89.96 & \textbf{90.50} & 90.38 & 90.39 & \textbf{90.58} & 90.45 & 90.45 & \textbf{90.50} & 90.32 & 90.34 & \textbf{90.45} & 90.30 & 90.30 \\
D29 & \textbf{87.29} & 87.27 & 87.31 & \textbf{87.58} & 87.56 & 87.56 & \textbf{87.62} & 87.56 & 87.55 & \textbf{87.65} & 87.53 & 87.50 & \textbf{87.74} & 87.70 & 87.70 \\
D30 & \textbf{96.63} & 96.59 & 96.59 & \textbf{97.70} & 97.68 & 97.68 & \textbf{97.82} & \textbf{97.82} & \textbf{97.82} & 97.87 & \textbf{97.89} & \textbf{97.89} & 97.95 & \textbf{97.96} & 97.95 \\
D31 & \textbf{91.59} & 91.56 & \textbf{91.59} & \textbf{91.58} & 91.53 & 91.56 & \textbf{91.59} & 91.48 & 91.52 & \textbf{91.62} & 91.58 & 91.53 & \textbf{91.74} & \textbf{91.74} & 91.70 \\
D32 & \textbf{68.15} & 67.91 & 67.91 & \textbf{68.41} & 68.19 & 68.18 & \textbf{68.40} & 68.20 & 68.19 & \textbf{68.36} & 68.00 & 68.14 & \textbf{68.02} & 67.54 & 67.79 \\
D33 & \textbf{80.11} & 79.74 & 79.70 & \textbf{80.54} & 80.25 & 80.21 & \textbf{80.70} & 80.30 & 80.28 & \textbf{81.22} & 80.67 & 80.52 & \textbf{81.30} & 80.65 & 80.53 \\
D34 & 96.16 & \textbf{96.18} & \textbf{96.18} & 96.38 & 96.40 & \textbf{96.40} & \textbf{96.42} & \textbf{96.42} & \textbf{96.42} & 96.40 & \textbf{96.41} & \textbf{96.42} & \textbf{96.44} & \textbf{96.44} & \textbf{96.46} \\
D35 & \textbf{99.85} & 99.25 & 99.25 & \textbf{99.85} & 99.25 & 99.25 & \textbf{99.85} & 99.25 & 99.25 & \textbf{99.85} & 99.85 & 99.25 & \textbf{99.85} & 99.85 & 99.25 \\
D36 & \textbf{100.00} & \textbf{100.00} & \textbf{100.00} & \textbf{100.00} & \textbf{100.00} & \textbf{100.00} & \textbf{100.00} & \textbf{100.00} & \textbf{100.00} & \textbf{100.00} & \textbf{100.00} & \textbf{100.00} & \textbf{100.00} & \textbf{100.00} & \textbf{100.00} \\
D37 & \textbf{92.88} & 92.78 & 92.78 & \textbf{92.97} & 92.78 & 92.78 & \textbf{92.92} & 92.74 & 92.73 & \textbf{92.79} & 92.71 & 92.65 & \textbf{92.86} & 92.70 & 92.70 \\
D38 & \textbf{80.94} & 80.81 & 80.81 & 81.01 & 80.92 & 80.92 & \textbf{81.29} & 81.17 & 81.17 & \textbf{81.31} & 81.13 & 81.18 & \textbf{81.24} & 81.11 & 81.15 \\
D39 & 99.63 & \textbf{99.64} & 99.63 & 99.62 & 99.62 & 99.62 & 99.62 & 99.62 & 99.62 & 99.61 & \textbf{99.63} & \textbf{99.63} & 99.61 & 99.62 & \textbf{99.64} \\
D40 & \textbf{86.83} & 86.64 & 86.66 & \textbf{86.79} & 86.46 & 86.49 & \textbf{86.76} & 86.46 & 86.46 & \textbf{86.76} & 86.61 & 86.51 & \textbf{86.88} & 86.71 & 86.64 \\
D41 & \textbf{74.32} & \textbf{74.32} & \textbf{74.32} & \textbf{78.84} & \textbf{78.84} & 78.83 & \textbf{82.26} & 82.23 & 82.23 & \textbf{82.97} & 82.90 & 82.92 & \textbf{83.53} & 83.50 & 83.52 \\
D42 & \textbf{57.80} & 56.95 & 56.91 & \textbf{57.80} & 57.14 & 57.18 & \textbf{57.89} & 57.61 & 57.64 & \textbf{57.98} & 57.48 & 57.68 & \textbf{57.89} & 57.70 & 57.84 \\
\midrule
Average & \textbf{86.78} & 86.68 & 86.68 & \textbf{87.01} & 86.91 & 86.91 & \textbf{87.17} & 87.08 & 87.07 & \textbf{87.26} & 87.16 & 87.15 & \textbf{87.30} & 87.19 & 87.18 \\
\bottomrule
\end{tabular}
}
\end{table}

Table \ref{tab:tolerance_comparison} studies the screening tolerance $\alpha$ under the heterogeneous (OOF) setting. Across all tolerance levels, BPE remains the top (or tied-top) method on the majority of datasets, demonstrating that the proposed weighting mechanism is robust to the choice of the screening aggressiveness.

\section{Conclusions}\label{7}
The BPE framework fundamentally represents a paradigm shift. It transitions from the traditional ensemble approach, which bases integration on the performance disparity among different models, to a novel perspective that bases integration on the variation of a single model's responses relative to its own internal behavioral profile.

In this work, the proposed BPE-Entropy algorithm should be viewed as a representative baseline of this general framework, rather than an attempt to exhaustively optimize the metric. The use of negative entropy and Gaussian perturbations is deliberately chosen to isolate and evaluate the efficacy of the core principles, allowing the theoretical properties of the BPE framework to be clearly observed without the interference of complex components. The effectiveness and robustness of the BPE-Entropy algorithm, serving as a baseline implementation of the BPE framework, are substantiated in a large number of experiments on real-world datasets.

\subsection{Concluding Remarks}

Based on the experimental results supported by appropriate statistical analyses, the following remarks are drawn:

\begin{itemize}
    \item This study indicates that scrutinizing the internal output behavior of models is critical for ensemble systems.
    
    \item Supported by both theoretical intuition and experimental results, the BPE framework offers a novel perspective for approaching the theoretical performance upper bound of ensemble systems.
    
    \item Due to the promising performance of the BPE-Entropy algorithm, we believe that this work will contribute to the \textit{refinement and advancement} of ensemble learning methodologies.
\end{itemize}

\subsection{Contributions}

In summary, the main contributions of this paper are as follows:

\begin{itemize}
    \item We formalize a novel framework based on the divergence between a model's predictive behavior and its internal behavioral baseline, theoretically elucidating the core intuition behind this mechanism.
    
    \item Based on the proposed BPE framework, we develop a baseline implementation, named BPE-Entropy, and empirically validate its effectiveness through extensive experiments.
\end{itemize}

\subsection{Limitations and future works}
Consistent with the ``No Free Lunch'' theorem in machine learning, the BPE algorithm is not universally optimal across all scenarios and encounters specific challenges. In particular, as a nascent ensemble perspective, it remains in an early stage of development. Based on the intrinsic characteristics of this framework, we propose the following potential directions for future research:

\begin{itemize}
    \item \textbf{Metrics for Behavioral Profiling:} A critical question remains: what should serve as the model's behavioral profile? In this study, the BPE framework utilizes the entropy of the predicted output as the profiling metric. However, whether entropy can fully and accurately characterize a model's behavioral properties remains an open question. We also explored using the margin between the Top-1 and Top-2 prediction probabilities as an alternative profile, which outperformed entropy in certain scenarios. Consequently, identifying a metric that accurately and comprehensively describes a model's behavioral characteristics constitutes a significant problem for future investigation.

    \item \textbf{Construction of Behavioral Profiles:} Currently, we construct the behavioral profile through a simple noise injection strategy on the training set. This serves as a foundational baseline implementation, as the primary focus of this work is to highlight the shift in ensemble perspective rather than to pursue the ultimate performance limit. Therefore, developing more sophisticated methods for constructing behavioral profiles will be key to raising the performance ceiling. A distinct advantage of the BPE framework over methods like DES and Stacking is that it does not require a validation set. This characteristic allows for various data augmentation techniques on the training set without the need for ground truth labels, thereby mitigating the limitations imposed by small-sample environments compared to traditional approaches.

    \item \textbf{Integration of Perspectives:} Traditional ensemble methods focus on evaluating how well a model performs relative to \textit{other models}, whereas the BPE framework focuses on evaluating how well a model performs relative to itself. These two ideologies are not mutually exclusive but rather complementary. Thus, investigating how to effectively integrate these two perspectives is a highly valuable research direction. Preliminary experiments conducted in this study suggest that fusing the weights derived from both ideologies can indeed achieve a higher performance upper bound in many scenarios.
\end{itemize}

\bibliographystyle{elsarticle-num} 
\bibliography{refs}               

\begin{thebibliography}{10}
\expandafter\ifx\csname url\endcsname\relax
  \def\url#1{\texttt{#1}}\fi
\expandafter\ifx\csname urlprefix\endcsname\relax\def\urlprefix{URL }\fi
\expandafter\ifx\csname href\endcsname\relax
  \def\href#1#2{#2} \def\path#1{#1}\fi

\bibitem{vanderlaan2007super}
M.~J. van~der Laan, E.~C. Polley, A.~E. Hubbard, Super learner, Statistical Applications in Genetics and Molecular Biology 6~(1) (2007) Article 25.

\bibitem{timmermann2006forecast}
A.~Timmermann, Forecast combinations, in: Handbook of Economic Forecasting, Vol.~1, Elsevier, 2006, pp. 135--196.

\bibitem{krogh1995neural}
A.~Krogh, J.~Vedelsby, Neural network ensembles, cross validation, and active learning, in: Advances in Neural Information Processing Systems, 1995, pp. 231--238.

\bibitem{dietterich2000ensemble}
T.~G. Dietterich, Ensemble methods in machine learning, Multiple Classifier Systems (2000) 1--15.

\bibitem{kittler1998combining}
J.~Kittler, M.~Hatef, R.~P. Duin, J.~Matas, On combining classifiers, IEEE Transactions on Pattern Analysis and Machine Intelligence 20~(3) (1998) 226--239.

\bibitem{ho1994decision}
T.~K. Ho, J.~J. Hull, S.~N. Srihari, Decision combination in multiple classifier systems, IEEE Transactions on Pattern Analysis and Machine Intelligence 16~(1) (1994) 66--75.

\bibitem{wolpert1992stacked}
D.~H. Wolpert, Stacked generalization, Neural Networks 5~(2) (1992) 241--259.

\bibitem{cruz2018dynamic}
R.~M. Cruz, R.~Sabourin, G.~D. Cavalcanti, Dynamic classifier selection: Recent advances and perspectives, Information Fusion 41 (2018) 195--216.

\bibitem{woods1997combination}
K.~Woods, W.~P. Kegelmeyer, K.~Bowyer, Combination of multiple classifiers using local accuracy estimates, IEEE Transactions on Pattern Analysis and Machine Intelligence 19~(4) (1997) 405--410.

\bibitem{ko2008dynamic}
A.~H. Ko, R.~Sabourin, A.~S. Britto~Jr, From dynamic classifier selection to dynamic ensemble selection, Pattern Recognition 41~(5) (2008) 1718--1731.

\bibitem{gal2016dropout}
Y.~Gal, Z.~Ghahramani, Dropout as a bayesian approximation: Representing model uncertainty in deep learning, in: International Conference on Machine Learning, PMLR, 2016, pp. 1050--1059.

\bibitem{lakshminarayanan2017simple}
B.~Lakshminarayanan, A.~Pritzel, C.~Blundell, Simple and scalable predictive uncertainty estimation using deep ensembles, in: Advances in Neural Information Processing Systems, Vol.~30, 2017.

\bibitem{wolpert1997no}
D.~H. Wolpert, W.~G. Macready, No free lunch theorems for optimization, IEEE Transactions on Evolutionary Computation 1~(1) (1997) 67--82.

\bibitem{shannon1948mathematical}
C.~E. Shannon, A mathematical theory of communication, The Bell System Technical Journal 27~(3) (1948) 379--423.

\bibitem{grandvalet2004semi}
Y.~Grandvalet, Y.~Bengio, Semi-supervised learning by entropy minimization, in: Advances in Neural Information Processing Systems, Vol.~17, 2004, pp. 529--536.

\bibitem{schapire1998boosting}
R.~E. Schapire, Y.~Freund, P.~Bartlett, W.~S. Lee, Boosting the margin: A new explanation for the effectiveness of voting methods, The Annals of Statistics 26~(5) (1998) 1651--1686.

\bibitem{breiman1996bagging}
L.~Breiman, Bagging predictors, Machine Learning 24~(2) (1996) 123--140.

\bibitem{breiman2001random}
L.~Breiman, Random forests, Machine Learning 45~(1) (2001) 5--32.

\bibitem{freund1997decision}
Y.~Freund, R.~E. Schapire, A decision-theoretic generalization of on-line learning and an application to boosting, Journal of Computer and System Sciences 55~(1) (1997) 119--139.

\bibitem{friedman2001greedy}
J.~H. Friedman, Greedy function approximation: a gradient boosting machine, The Annals of Statistics (2001) 1189--1232.

\bibitem{chen2016xgboost}
T.~Chen, C.~Guestrin, Xgboost: A scalable tree boosting system, in: Proceedings of the 22nd ACM SIGKDD International Conference on Knowledge Discovery and Data Mining, 2016, pp. 785--794.

\bibitem{ke2017lightgbm}
G.~Ke, Q.~Meng, T.~Finley, T.~Wang, W.~Chen, W.~Ma, Q.~Ye, T.-Y. Liu, Lightgbm: A highly efficient gradient boosting decision tree, in: Advances in Neural Information Processing Systems, 2017, pp. 3146--3154.

\bibitem{prokhorenkova2018catboost}
L.~Prokhorenkova, G.~Gusev, A.~Vorobev, A.~V. Dorogush, A.~Gulin, Catboost: unbiased boosting with categorical features, in: Advances in Neural Information Processing Systems, 2018, pp. 6638--6648.

\bibitem{cortes1995support}
C.~Cortes, V.~Vapnik, Support-vector networks, Machine Learning 20~(3) (1995) 273--297.

\bibitem{cover1967nearest}
T.~Cover, P.~Hart, Nearest neighbor pattern classification, IEEE Transactions on Information Theory 13~(1) (1967) 21--27.

\bibitem{cox1958regression}
D.~R. Cox, The regression analysis of binary sequences, Journal of the Royal Statistical Society: Series B (Methodological) 20~(2) (1958) 215--232.

\bibitem{woloszynski2011probabilistic}
T.~Woloszynski, M.~Kurzynski, A probabilistic model of classifier competence for dynamic ensemble selection, Pattern Recognition 44~(10-11) (2011) 2656--2668.

\bibitem{giacinto2001design}
G.~Giacinto, F.~Roli, Design of effective neural network ensembles for image classification purposes, Image and Vision Computing 19~(9-10) (2001) 699--707.

\bibitem{cruz2015meta}
R.~M. Cruz, R.~Sabourin, G.~D. Cavalcanti, T.~I. Ren, Meta-des: A dynamic ensemble selection framework using meta-learning, Pattern Recognition 48~(5) (2015) 1925--1935.

\bibitem{cruz2019fire}
R.~M. Cruz, D.~V. Oliveira, G.~D. Cavalcanti, R.~Sabourin, Fire-des++: Enhanced online pruning of base classifiers for dynamic ensemble selection, Pattern Recognition 85 (2019) 149--160.

\bibitem{Beyer1999NN}
K.~Beyer, J.~Goldstein, R.~Ramakrishnan, U.~Shaft, When is ``nearest neighbor'' meaningful?, in: Proceedings of the International Conference on Database Theory (ICDT), Vol. 1540 of Lecture Notes in Computer Science, Springer, 1999, pp. 217--235.

\bibitem{hastie2009elements}
T.~Hastie, R.~Tibshirani, J.~Friedman, The Elements of Statistical Learning: Data Mining, Inference, and Prediction, Springer Science \& Business Media, 2009.

\bibitem{jaynes1957information}
E.~T. Jaynes, Information theory and statistical mechanics, Physical review 106~(4) (1957) 620.

\bibitem{vanschoren2013openml}
J.~Vanschoren, J.~N. Van~Rijn, B.~Bischl, L.~Torgo, Openml: networked science in machine learning, ACM SIGKDD Explorations Newsletter 15~(2) (2013) 49--60.

\bibitem{geurts2006extremely}
P.~Geurts, D.~Ernst, L.~Wehenkel, Extremely randomized trees, Machine Learning 63~(1) (2006) 3--42.

\bibitem{fisher1936use}
R.~A. Fisher, The use of multiple measurements in taxonomic problems, Annals of Eugenics 7~(2) (1936) 179--188.

\bibitem{rumelhart1986learning}
D.~E. Rumelhart, G.~E. Hinton, R.~J. Williams, Learning representations by back-propagating errors, Nature 323~(6088) (1986) 533--536.

\bibitem{zhang2004optimality}
H.~Zhang, The optimality of naive bayes, in: Proceedings of the Seventeenth International Florida Artificial Intelligence Research Society Conference, Vol.~2, 2004, pp. 562--567.

\bibitem{caruana2004ensemble}
R.~Caruana, A.~Niculescu-Mizil, G.~Crew, A.~Ksikes, Ensemble selection from libraries of models, in: Proceedings of the twenty-first international conference on Machine learning, ACM, 2004, p.~18.

\bibitem{martinezmunoz2006pruning}
G.~Mart{\'{i}}nez-Mu{\~{n}}oz, A.~Su{\'{a}}rez, Pruning in ordered bagging ensembles, in: Proceedings of the 23rd International Conference on Machine Learning, ACM, 2006, pp. 609--616.

\bibitem{martinezmunoz2007boosting}
G.~Mart{\'{i}}nez-Mu{\~{n}}oz, A.~Su{\'{a}}rez, Using boosting to prune bagging ensembles, Pattern Recognition Letters 28~(1) (2007) 156--165.

\bibitem{zhou2002many}
Z.-H. Zhou, J.~Wu, W.~Tang, Ensembling neural networks: many could be better than all, Artificial Intelligence 137~(1--2) (2002) 239--263.

\bibitem{clemen1989combining}
R.~T. Clemen, Combining forecasts: A review and annotated bibliography, International Journal of Forecasting 5~(4) (1989) 559--583.

\bibitem{zhang2025des}
Z.-L. Zhang, Y.-H. Zhu, X.-G. Luo, Des-as: Dynamic ensemble selection based on algorithm shapley, Pattern Recognition 157 (2025) 110899.

\bibitem{desouto2008dcs}
M.~C.~P. de~Souto, R.~G.~F. Soares, A.~Santana, A.~M.~P. Canuto, Empirical comparison of dynamic classifier selection methods based on diversity and accuracy for building ensembles, in: 2008 IEEE International Joint Conference on Neural Networks (IEEE World Congress on Computational Intelligence), IEEE, 2008, pp. 1480--1487.

\end{thebibliography}

\end{document}